\documentclass{article}

\usepackage{amssymb}
\usepackage{amsfonts}

\usepackage{tgtermes}
\usepackage{amsmath}
\usepackage{amsthm}  %
\usepackage{scalefnt,letltxmacro}
\LetLtxMacro{\oldtextsc}{\textsc}
\renewcommand{\textsc}[1]{\oldtextsc{\scalefont{1.10}#1}}
\usepackage[scaled=0.92]{PTSans}

\usepackage[usenames,dvipsnames]{xcolor}
\usepackage{soul}
\definecolor{shadecolor}{gray}{0.9}

\usepackage{afterpage}
\usepackage{framed}
\usepackage{nicefrac}
\usepackage{bm}
\usepackage{paralist}

\usepackage[colorinlistoftodos,
           prependcaption,
           textsize=small,
           backgroundcolor=yellow,
           linecolor=lightgray,
           bordercolor=lightgray]{todonotes}

\usepackage{lineno}

\usepackage{ragged2e}

\DeclareRobustCommand{\parhead}[1]{\textbf{#1}~}

\usepackage{graphicx}
\usepackage[labelfont=it,labelsep=period]{caption}
\usepackage[format=hang]{subcaption}
\usepackage{wrapfig}
\usepackage{placeins}

\usepackage{booktabs}
\usepackage{arydshln} %
\usepackage{multirow}

\usepackage{algorithm}
\usepackage{listings}
\usepackage{fancyvrb}
\fvset{fontsize=\normalsize}
\usepackage[noend]{algpseudocode}

\usepackage[colorlinks,linktoc=all]{hyperref}
\usepackage[all]{hypcap}
\hypersetup{citecolor=Violet}
\hypersetup{linkcolor=black}
\hypersetup{urlcolor=MidnightBlue}
\usepackage{url}

\usepackage[nameinlink]{cleveref}
\creflabelformat{equation}{#1#2#3}
\crefname{equation}{eq.}{eqs.}  
\Crefname{equation}{Eq.}{Eqs.}
\usepackage[acronym,smallcaps,nowarn]{glossaries}
\usepackage{listings}
\lstdefinestyle{alp_style}{
    commentstyle=\color{OliveGreen},
    numberstyle=\tiny\color{black!60},
    stringstyle=\color{BrickRed},
    basicstyle=\ttfamily\scriptsize,
    breakatwhitespace=false,
    breaklines=true,
    captionpos=b,
    keepspaces=true,
    numbers=none,
    numbersep=5pt,
    showspaces=false,
    showstringspaces=false,
    showtabs=false,
    tabsize=2
}
\usepackage{lipsum}

\newcommand{\g}{\, | \,}

\DeclareMathOperator{\Cov}{Cov}
\DeclareMathOperator{\E}{\mathbb{E}}
\DeclareMathOperator{\Var}{Var}
\newcommand{\argmin}{\operatornamewithlimits{argmin}}

\def\bE{{\mathbb E}}

\def\f0{{\mathbf 0}}

\def\w{{\mathbf w}}
\def\md{{\mathrm d}}

\newtheorem{defn}{Definition}

\newtheorem{lemma}{Lemma}
\newtheorem{proposition}{Proposition}
\newtheorem{corollary}{Corollary}

\theoremstyle{definition}

\newtheorem{remark}{Remark}

\newacronym{ADVI}{advi}{automatic differentiation variational inference}

\newacronym{BBVI}{bbvi}{black-box variational inference}

\newacronym{CDF}{cdf}{cumulative distribution function}

\newacronym{DVAE}{dvae}{discrete variational autoencoder}

\newacronym{ELBO}{elbo}{evidence lower bound}
\newacronym{EM}{em}{expectation maximization}

\newacronym{HMC}{hmc}{{H}amiltonian {M}onte {C}arlo}

\newacronym{KL}{kl}{{K}ullback-{L}eibler}

\newacronym{LDA}{lda}{latent {D}irichlet allocation}
\newacronym{LSTM}{lstm}{long short-term memory}

\newacronym{MAP}{map}{\emph{maximum a posteriori}}
\newacronym{MCMC}{mcmc}{{M}arkov chain {M}onte {C}arlo}

\newacronym{SVI}{svi}{stochastic variational inference}

\newacronym{VAE}{vae}{variational autoencoder}
\newacronym{VEM}{vem}{variational expectation maximization}

\usepackage[final]{neurips_2020}

\usepackage[utf8]{inputenc} %
\usepackage[T1]{fontenc}    %
\usepackage{hyperref}       %
\usepackage{url}            %
\usepackage{booktabs}       %
\usepackage{amsfonts}       %
\usepackage{nicefrac}       %
\usepackage{microtype}      %

\usepackage[symbol]{footmisc}

\hypersetup{citecolor=Violet}
\hypersetup{linkcolor=black}
\hypersetup{urlcolor=MidnightBlue}

\title{VarGrad: A Low-Variance Gradient Estimator for Variational Inference}

\author{%
    Lorenz Richter\thanks{Equal contribution.}\\
    Freie Universit\"{a}t Berlin\\
    BTU Cottbus-Senftenberg\\ %
    dida Datenschmiede GmbH\\
    \texttt{lorenz.richter@fu-berlin.de}
    \And
    Ayman Boustati$^*$\\
    University of Warwick\\
    \texttt{a.boustati@warwick.ac.uk}
    \And
    Nikolas N\"{u}sken\\
    Universit\"{a}t Potsdam\\
    \texttt{nuesken@uni-potsdam.de}
    \And
    Francisco J.\ R.\ Ruiz\\
    DeepMind\\
    \texttt{franrruiz@google.com}
    \And
    \"Omer Deniz Akyildiz\\
    University of Warwick\\
    The Alan Turing Institute\\
    \texttt{omer.akyildiz@warwick.ac.uk}\\
}

\begin{document}

\maketitle

\begin{abstract}

We analyse the properties of an unbiased gradient estimator of the \gls{ELBO} for variational inference, based on the score function method with leave-one-out control variates. We show that this gradient estimator can be obtained using a new loss, defined as the variance of the log-ratio between the exact posterior and the variational approximation, which we call the \emph{log-variance loss}. Under certain conditions, the gradient of the log-variance loss equals the gradient of the (negative) \gls{ELBO}. We show theoretically that this gradient estimator, which we call \emph{VarGrad} due to its connection to the log-variance loss, exhibits lower variance than the score function method in certain settings, and that the leave-one-out control variate coefficients are close to the optimal ones. We empirically demonstrate that VarGrad offers a favourable variance versus computation trade-off compared to other state-of-the-art estimators on a discrete \gls{VAE}.
\end{abstract}

\glsresetall

\section{Introduction}
\label{sec:introduction}

Estimating the gradient of the expectation of a function is a problem with applications in many areas of machine learning, ranging from variational inference to reinforcement learning \citep{mohamed2019monte}. Different gradient estimators lead to different algorithms; two examples of estimators are the score function gradient \citep{williams1992simple} and the reparameterisation gradient \citep{kingma2014auto,rezende2014stochastic,titsias2014doubly}. Many recent works develop new estimators with different properties (such as their variance); see \Cref{sec:related} for a review.

We focus on variational inference ({VI}), where the goal is to approximate the posterior distribution $p(z \g x)$ of a model $p(x,z)$, where $x$ denotes the observations and $z$ refers to the latent variables of the model \citep{jordan1999introduction, blei2017variational}. {VI} approximates the posterior using a parameterised family of distributions $q_\phi(z)$, and finds the parameters $\phi$ by minimising the \gls{KL} divergence from $q_\phi(z)$ to $p(z \g x)$, i.e., $\textrm{\textsc{kl}}(q_\phi(z) \;||\; p(z \g x))$. Since the \gls{KL} is intractable, VI solves instead an equivalent problem that maximises the \gls{ELBO}, given by
\begin{equation}\label{eq:elbo}
    \textrm{\textsc{elbo}}(\phi) = \E_{q_{\phi}}\left[ \log \frac{p(x, z)}{q_{\phi}(z)}  \right].
\end{equation}
Thus, VI casts the inference problem as an optimisation problem, which can be solved with stochastic optimisation tools when the \gls{ELBO} is not available in closed form. In particular, VI forms a Monte Carlo estimator of the gradient of the \gls{ELBO}, $\nabla_\phi \textrm{\textsc{elbo}}(\phi)$.

In this paper, we analyse a multi-sample estimator of the gradient of the \gls{ELBO}. In particular, we focus on an estimator first introduced by \citet{salimans2014onusing} and \citet{kool2019buy}, which is based on the score function method \citep{williams1992simple} with leave-one-out control variates.

We first show the connection between this estimator and an alternative divergence measure between the variational distribution $q_\phi(z)$ and the exact posterior $p(z \g x)$. This divergence, which is different from the standard \gls{KL} used in variational inference, is defined as the variance, under some arbitrary distribution $r(z)$, of the log-ratio $\log \frac{q_{\phi}(z)}{p(z\g x)}$. We refer to this divergence as the \emph{log-variance loss}. Inspired by \citet{nusken2020solving}, we show that we recover the gradient estimator of \citet{salimans2014onusing} and \citet{kool2019buy} by taking the gradient with respect to the variational parameters $\phi$ of the log-variance loss and evaluating the result at $r(z)=q_{\phi}(z)$. Due to this property, we refer to the gradient estimator as \emph{VarGrad}. This property also suggests a simple algorithm for computing the gradient estimator, based on differentiating through the log-variance loss.

We then study the relationship between VarGrad and the score function estimator \citep{williams1992simple,carbonetto2009stochastic,paisley2012variational,ranganath2014black} with optimal control variate coefficients. We show that the control variate coefficients of VarGrad are close to the (intractable) optimal coefficients. Indeed, we show both theoretically and empirically that the difference between both is small in many cases; for example when the \gls{KL} from $q_{\phi}(z)$ to the posterior is either small or large, which is generally the case in the late and early stages of the optimisation, respectively. This explains the success of the VarGrad estimator in a variety of settings \citep{kool2019buy,kool2020estimating}.

Since it is based on the score function, VarGrad is a black-box, general purpose estimator because it makes no assumptions on the model $p(x,z)$, such as differentiability with respect to the latent variables $z$. It introduces no additional parameters to be tuned and it is not computationally expensive. In  \Cref{sec:experiments}, we show empirically that VarGrad exhibits a favourable variance versus computation trade-off compared to other unbiased gradient estimators, including the score function gradient with control variates \citep{williams1992simple,ranganath2014black}, \textsc{rebar} \citep{tucker2017rebar}, \textsc{relax} \citep{grathwohl2018backpropagation}, and \textsc{arm} \citep{yin2018arm}.

\section{Background}
\label{sec:background}

In this section, we introduce the notation and review one of the most relevant estimators in VI: the score function method. We also review its improved version based on leave-one-out control variates.

Consider a probabilistic model $p(x,z)$, where $z$ denotes the latent variables and $x$ are the observations. We are interested in computing the posterior $p(z\g x)=p(x,z)/p(x)$, where $p(x) = \int p(x,z) \;\md z$ is the marginal likelihood. 
For most models of interest, the posterior is intractable due to the intractability of the marginal likelihood, and we resort to an approximation.

Variational inference approximates the posterior $p(z\g x)$ with a parameterised family of distributions $q_\phi(z)$ (with $\phi \in \Phi$), called variational family. Variational inference finds the parameters $\phi^*$ that minimise the \gls{KL} divergence, $\phi^* = \argmin_{\phi \in \Phi} \textrm{\textsc{kl}}\left(q_\phi(z) \;||\; p(z \g x) \right)$. This optimisation problem is intractable because the \gls{KL} itself depends on the intractable posterior. Variational inference sidesteps this problem by maximising instead the \gls{ELBO} defined in \Cref{eq:elbo}, which is a lower bound on the marginal likelihood, since $\log p(x) = \textrm{\textsc{elbo}}(\phi) 
+ \textrm{\textsc{kl}}\left(q_\phi(z) \;||\; p(z \g x) \right)$. As the expectation in \Cref{eq:elbo} is typically intractable, variational inference uses stochastic optimisation to maximise the \gls{ELBO}. In particular, it forms unbiased Monte Carlo estimators of the gradient $\nabla_{\phi}\textrm{\textsc{elbo}}(\phi)$.

We next review the score function method, a Monte Carlo estimator commonly used in variational inference. Instead of the \gls{ELBO}, we focus on the gradients of the \gls{KL} divergence with respect to the variational parameters $\phi$. These gradients are equal to the gradients of the negative \gls{ELBO} because the marginal likelihood $p(x)$ does not depend on $\phi$; that is, $\nabla_{\phi} \textrm{\textsc{kl}}\left(q_\phi(z) \;||\; p(z \g x) \right) = -\nabla_{\phi}\textrm{\textsc{elbo}}(\phi)$.

The score function estimator \citep{williams1992simple,carbonetto2009stochastic,paisley2012variational,ranganath2014black}, also known as Reinforce, expresses the gradient as an expectation that depends on the log-ratio ${q_\phi(z)}/{p(x, z)}$ weighted by the score function $\nabla_\phi \log q_\phi(z)$. The resulting estimator is %
\begin{equation}
 \label{eq:score_function}
 \nabla_\phi \textrm{\textsc{kl}}\left[q_\phi(z) \;||\; p(z \g x) \right] \approx \widehat{g}_{\text{Reinforce}}(\phi) = \frac{1}{S} \sum_{s=1}^S \log\!\left(\frac{q_\phi(z^{(s)})}{p(x, z^{(s)})} \right)\! \nabla_\phi \log q_\phi(z^{(s)}),
\end{equation}
where $z^{(s)} \overset{\textrm{i.i.d.}}{\sim} q_\phi(z)$. Due to its high variance, the score function estimator requires additional tricks in practice; \citet{ranganath2014black} use Rao-Blackwellization and control variates. The control variates are multiples of the score function, $a\odot \nabla_\phi \log q_\phi(z)$, where $\odot$ denotes the Hadamard (element-wise) product and the coefficient $a$ is chosen to minimise the estimator variance.

\citet{salimans2014onusing} and \citet{kool2019buy} leverage the multi-sample estimator by using $S-1$ samples to compute the control variate coefficient $a$ and then average over the resulting estimators, obtaining a leave-one-out estimator
\begin{align}
    \label{eq:kool_estimator}
	\widehat{g}_{\text{LOO}}(\phi) = \frac{1}{S-1} \!\! \left( \sum_{s=1}^S f_{\phi}(z^{(s)}) \nabla_\phi \log q_\phi(z^{(s)})  - \bar{f}_{\phi}\sum_{s=1}^S \nabla_\phi \log q_\phi (z^{(s)}) \! \right)\!\!, \quad z^{(s)} \overset{\textrm{i.i.d.}}{\sim} q_{\phi}(z),
\end{align}
where for simplicity of notation we have defined $f_\phi(z)$ as the log-ratio $\log \frac{q_\phi(z)}{p(x, z)}$ and $\bar{f}_\phi$ as its empirical average (which is a multi-sample Monte Carlo estimate of the negative \gls{ELBO}), i.e.,
\begin{equation}
    \label{eq:definition_of_f}
    f_{\phi}(z) = \log \frac{q_\phi(z)}{p(x, z)}
    \quad \textrm{and} \quad
    \bar{f}_{\phi} = \frac{1}{S} \sum_{s=1}^{S}f_{\phi}(z^{(s)})\approx -\textrm{\textsc{elbo}}(\phi).
\end{equation}
The score function method makes no assumptions on the model $p(x, z)$ or the distribution $q_{\phi}(z)$; the only requirements are to be able to sample from $q_{\phi}(z)$ and to evaluate $\log q_\phi(z)$ and $\log p(x, z)$.

\section{The Log-Variance Loss and its Connection to VarGrad}
\label{sec:vargrad}

In this section, we show the connection between the leave-one-out estimator in \Cref{eq:kool_estimator} and a novel divergence, which we call the log-variance loss. We introduce the log-variance loss in \Cref{subsec:log_variance_loss} and show its connection to \Cref{eq:kool_estimator} in \Cref{subsec:vargrad}. We refer to the estimator in \Cref{eq:kool_estimator} as \emph{VarGrad}.

\subsection{The Log-Variance Loss}
\label{subsec:log_variance_loss}
The log-variance loss is defined as the variance, under some arbitrary distribution $r(z)$, of the log-ratio $\log \frac{q_{\phi}(z)}{p(z \g x)}$. It has the property of reproducing the gradients of the \gls{KL} divergence under certain conditions (see \Cref{proposition: equivalence log-variance relative entropy} for details). We next give the precise definition of the loss.
\begin{defn} For a given distribution $r(z)$, the log-variance loss $\mathcal{L}_r(\cdot)$ is given by
\begin{align}
    \mathcal{L}_{r}(q_\phi(z) \;||\; p(z \g x)) &= \frac{1}{2} \textnormal{Var}_{r}\left(\log\left(\frac{q_\phi(z)}{p(z \g x)}\right)\right). %
    \label{eq:log_variance_loss}
\end{align}
\end{defn}
We refer to the distribution $r(z)$ as the \emph{reference distribution} under which the discrepancy between $q_\phi(z)$
and the posterior $p(z \g x)$ is computed. When the support of the reference distribution contains the supports of $q_\phi(z)$ and $p(z \g x)$, \Cref{eq:log_variance_loss} is a divergence;\footnote{%
More technically, as we assume that $r(z)$, $p(z\g x)$, and $q_{\phi}(z)$ admit densities, it follows that measure-zero sets of $r(z)$ are necessarily measure-zero sets of $p(z\g x)$ and $q_{\phi}(z)$, implying that the divergence is well defined.}
it is zero if and only if $q_\phi(z) = p(z \g x)$.
The factor $1/2$ in \Cref{eq:log_variance_loss} is only included because it simplifies some expressions later in this section.

We next show that the gradient of the log-variance loss and the gradient of the standard \gls{KL} divergence coincide under certain conditions. In particular, taking the gradient of \Cref{eq:log_variance_loss} with respect to the variational parameters $\phi$ \emph{and then} evaluating the result for a reference distribution $r(z)=q_{\phi}(z)$ gives the gradient of the \gls{KL}. This property is detailed in \Cref{proposition: equivalence log-variance relative entropy}.
\begin{proposition}
    \label{proposition: equivalence log-variance relative entropy}
    The gradient with respect to $\phi$ of the log-variance loss, evaluated at $r(z)=q_{\phi}(z)$, equals the gradient of the \gls{KL} divergence,
	\begin{equation}
    	\nabla_\phi \mathcal{L}_r (q_\phi(z) \;||\; p(z \g x)) \Big\vert_{r = q_\phi} = \nabla_\phi \textrm{\textsc{kl}}(q_\phi(z) \;||\; p(z \g x)).
    	\label{eqn:PropositionEqn}
	\end{equation}
\end{proposition}
\begin{proof}
See \Cref{proof:PropEquivalence}.
\end{proof}

\Cref{proposition: equivalence log-variance relative entropy} implies that we can estimate the gradient of the \gls{KL} divergence by estimating instead the gradient of the log-variance loss.

\begin{remark}
\label{remark: dont differentiate through r} 
The result in \Cref{proposition: equivalence log-variance relative entropy} is obtained by setting $r(z)=q_{\phi}(z)$ \emph{after} taking the gradient with respect to $\phi$. The same result does not hold if we set $r(z)=q_{\phi}(z)$ before differentiating.
\end{remark}

\subsection{VarGrad: Derivation of the Gradient Estimator from the Log-Variance Loss}
\label{subsec:vargrad}

The leave-one-out estimator in \Cref{eq:kool_estimator} \citep{salimans2014onusing,kool2019buy} is connected to the log-variance loss from \Cref{subsec:log_variance_loss} through \Cref{proposition: equivalence log-variance relative entropy}.
Firstly, note that the log-variance loss is intractable as it depends on the posterior $p(z \g x)$. However, since the marginal likelihood $p(x)$ has zero variance, it can be dropped from the definition in \Cref{eq:log_variance_loss}, yielding 
\begin{equation}
    \mathcal{L}_{r}(q_\phi(z) \;||\; p(z \g x)) = %
    \frac{1}{2} \text{Var}_{r}\left(\log\left(\frac{q_\phi(z)}{p(x, z)}\right)\right) =
    \frac{1}{2} \text{Var}_{r}\left(f_\phi(z) \right),
\end{equation}
where $f_\phi(z)$ is defined in \Cref{eq:definition_of_f}.

Next, we build the estimator of the log-variance loss as the empirical variance of $S$ Monte Carlo samples,
\begin{equation}
    \label{eq:log_variance_estimator}
    \mathcal{L}_{r}(q_\phi(z) \;||\; p(z \g x)) \approx \frac{1}{2(S-1)} \sum_{s=1}^{S} \left( f_{\phi}(z^{(s)}) - \bar{f}_{\phi} \right)^2,\quad z^{(s)} \overset{\textrm{i.i.d.}}{\sim} r(z).
\end{equation}
Applying \Cref{proposition: equivalence log-variance relative entropy} by differentiating through \Cref{eq:log_variance_estimator}, we arrive at the VarGrad estimator, $\widehat{g}_{\text{VarGrad}}(\phi) = \widehat{g}_{\text{LOO}}(\phi) \approx \nabla_\phi \textrm{\textsc{kl}}\left(q_\phi(z) \;||\; p(z \g x) \right)$, where
\begin{align}
\label{eq:vargrad}
	\widehat{g}_{\text{VarGrad}}(\phi) = \frac{1}{S-1} \!\! \left( \sum_{s=1}^S f_{\phi}(z^{(s)}) \nabla_\phi \log q_\phi(z^{(s)})  - \bar{f}_{\phi}\sum_{s=1}^S \nabla_\phi \log q_\phi (z^{(s)}) \! \right)\!\! 
\end{align}
and  $z^{(s)} \overset{\textrm{i.i.d.}}{\sim} q_{\phi}(z)$. 

The expression for VarGrad in \Cref{eq:vargrad} is identical to that of the leave-one-out estimator in \Cref{eq:kool_estimator}. Thus, VarGrad is an unbiased estimator of the gradient of the \gls{KL} (and equivalently the gradient of the \gls{ELBO}). From a probabilistic programming perspective, setting the reference $r(z)=q_{\phi}(z)$ \textit{after} differentiating w.r.t. $\phi$ amounts to sampling $z^{(s)} \sim q_{\phi}(z)$ and detaching the resulting samples from the computational graph. This suggests a novel algorithmic procedure, given in \Cref{alg:vargrad}. Its implementation is simple: we only need the samples $z^{(s)}\sim q_{\phi}(z)$ and apply the \verb+stop_gradient+ operator, evaluate the log-ratio $f_{\phi}(z^{(s)})$ for each sample, and then differentiate through the empirical variance of this log-ratio.

\begin{algorithm}[th]
  \renewcommand{\algorithmicrequire}{\textbf{Input:}}
  \renewcommand{\algorithmicensure}{\textbf{Output:}}
  \caption{Pseudocode for VarGrad\label{alg:vargrad}}
  \begin{algorithmic}
    \Require{Variational parameters $\phi$, data $x$}
    \For{$s=1, \ldots,  S$}
      \State $z^{(s)} \gets \verb+sample+(q_\phi(\cdot))$ \Comment{Sample from the approximate posterior}
      \State $z^{(s)} \gets \verb+stop_gradient+(z^{(s)})$ \Comment{Detach the samples from the computational graph}
      \State $f_\phi^{(s)} \gets \log q_\phi(z^{(s)}) - \log p(x, z^{(s)})$ \Comment{An estimate of the negative \acrshort{ELBO}}
    \EndFor
    \State $\widehat{\mathcal{L}} \gets$ $\frac{1}{2}\verb+Variance+(\{f_\phi^{(s)} \}_{s=1}^{S}$)
    \Comment{An estimate of the log-variance loss}
    \State \Return{\verb+grad+$(\widehat{\mathcal{L}})$}
    \Comment{Differentiate through the loss w.r.t.\ $\phi$}
  \end{algorithmic}
\end{algorithm}

\section{Analytical Results}
\label{sec:analytical}

In this section we study the properties of $\widehat{g}_{\text{VarGrad}}$ in 
comparison to other estimators based on the score function method. In \Cref{sec:CV}, we analyse the difference $\delta^{\text{CV}}$ between the control variate coefficient of VarGrad (called $a^{\text{VarGrad}}$) and the optimal one. The former can be approximated cheaply and unbiasedly, while a standard Monte Carlo estimator of the latter is biased and often exhibits high variance. Furthermore, we establish that the difference $\delta^{\text{CV}}$ is negligible in certain settings, in particular when $\textrm{\textsc{kl}}(q_\phi(z) \;||\; p(z \g x))$ is either very large or close to zero; thus in these settings the control variate coefficient of VarGrad is close to the optimal coefficient.
In \Cref{sec:variance} we show that a simple relation between $\delta^{\text{CV}}$ and the \gls{ELBO} is sufficient to \emph{guarantee} that $\widehat{g}_{\text{VarGrad}}$ has lower variance than $\widehat{g}_{\text{Reinforce}}$ when the number of Monte Carlo samples is large enough.

\subsection{Analysis of the Control Variate Coefficients}
\label{sec:CV}
As alluded to in \Cref{sec:background}, \citet{ranganath2014black} proposed to modify $\widehat{g}_{\text{Reinforce}}$ using a score function control variate, that is,
\begin{equation}
    \label{eq:g_cv}
    \widehat{g}_{\text{CV}}(\phi) = \widehat{g}_{\text{Reinforce}}(\phi) - a \odot \left( \frac{1}{S} \sum_{s=1}^S \nabla_{\phi} \log q_{\phi} (z^{(s)})\right),
\end{equation}
where $a$ is a vector chosen so as to reduce the variance of the estimator. We recover VarGrad (\Cref{eq:vargrad}), up to a factor of proportionality, by setting the control variate coefficient $a= \bar{f}_{\phi} \mathbf{1}$ in \Cref{eq:g_cv}, where $\mathbf{1}$ is a vector of ones. The proportionality relation is $\frac{S-1}{S} \widehat{g}_{\text{CV}} = \widehat{g}_{\text{VarGrad}}$.
In terms of variance reduction, the coefficients of the optimal $a^*$ are given by
\begin{equation}
\label{eqn: optimal control variate}
a_i^* =  \frac{\Cov_{q_\phi} \left(f_{\phi} \partial_{\phi_i} \log q_\phi, \partial_{\phi_i} \log q_\phi \right)}{\Var_{q_\phi} \left( \partial_{\phi_i} \log q_\phi \right)}.
\end{equation}
We next show that the coefficients of VarGrad, $a^{\text{VarGrad}}$, are close to the optimal coefficients $a^*$. For this, we first relate $a^{\text{VarGrad}}$ to $a^*$ in \Cref{lemma: optimal control variate decomposition}.

\begin{lemma}
\label{lemma: optimal control variate decomposition}
We can write the optimal control variate coefficient as the expected value of $a^{\text{VarGrad}}$ plus a \emph{control variate correction} term $\delta^{\text{CV}}$, i.e.,
\begin{equation}
a^* = \mathbb{E}_{q_\phi}[a^{\text{VarGrad}}] + \delta^{\text{CV}} = -\mathrm{ELBO}(\phi) + \delta^{\text{CV}}, 
\end{equation}
where $a^{\text{VarGrad}} = \bar{f}_\phi$ and the components of the  correction term $\delta^{\text{CV}}$ are given by
\begin{equation}
\label{eqn: delta_CV}
\delta^{\text{CV}}_i = \frac{\Cov_{q_\phi}\left( f_\phi,\left(\partial_{\phi_i} \log q_\phi\right)^2\right)}{\Var_{q_\phi} \left( \partial_{\phi_i} \log q_\phi \right)}.
\end{equation}
\end{lemma}

\begin{proof}
See \Cref{proof: optimal control variate decomposition}.
\end{proof}

According to \Cref{lemma: optimal control variate decomposition}, the difference between the optimal control variate coefficient and the (expected) VarGrad coefficient is equal to the correction $\delta^{\text{CV}}$. We hypothesise that direct Monte Carlo estimation of $\delta^{\text{CV}}$ in \Cref{eqn: delta_CV} (or similarly for \Cref{eqn: optimal control variate}) suffers from high variance because it takes the form of a fraction\footnote{
    Monte Carlo estimators of fractions are not straightforward. As a simple example, consider the ratio of two independent Gaussian random variables, each with zero mean and unit variance. The ratio follows a Cauchy distribution, which has infinite variance. 
}
(see for instance \Cref{app:gaussians}).
Moreover, estimating \Cref{eqn: delta_CV} by taking the ratio of two Monte Carlo estimators gives a biased estimate.

We next show that in certain settings the correction term $\delta^{\text{CV}}$
becomes negligible, implying that $\widehat{g}_{\text{VarGrad}}$ and $\widehat{g}_{\text{Reinforce}}$ equipped with the optimal control variate coefficients behave almost identically. We provide empirical evidence of this finding in \Cref{sec:experiments} (and in \Cref{app:gaussians} for the Gaussian case).

\begin{proposition}
[$\delta^{\text{CV}}$ is small in comparison to $\mathbb{E}_{q_\phi}[a^{\text{VarGrad}}{]}$ if the \gls{KL} divergence between $q_\phi(z)$ and $p(z \g x)$ is large or small]
\label{prop:delta small}
Assume that $q_\phi(z)$ has lighter tails than the posterior $p(z \g x)$, in the sense that there exists a constant $C>0$ such that
\begin{equation}
\label{eq:underspread}
\sup_{z} \frac{q_\phi(z)}{p(z \g x)} < C.
\end{equation}
Furthermore, define the kurtosis of the score function,
\begin{equation}
\mathrm{Kurt}[\partial_{\phi_i} \log q_\phi] = \frac{\mathbb{E}_{q_\phi}[(\partial_{\phi_i} \log q_\phi)^4]}{(\mathbb{E}_{q_\phi}[(\partial_{\phi_i} \log q_\phi)^2])^2},
\end{equation}
and assume that it is bounded, $\mathrm{Kurt}[\partial_{\phi_i} \log q_\phi] < \infty$. Then, the ratio between the control variate correction $\delta^{\text{CV}}$ and the expected control variate coefficient of VarGrad can be upper bounded by
\begin{equation}
\label{eq:cor bound}
\left\vert  \frac{\delta^{\text{CV}}_i}{\mathbb{E}_{q_\phi}[a^{\text{VarGrad}}]}\right\vert \le \frac{2\sqrt{C \, \mathrm{Kurt}[\partial_{\phi_i} \log q_\phi]}}{\left\vert \sqrt{\textrm{\textsc{kl}}(q_\phi(z) \;||\; p(z \g x))} -\frac{ \log p(x)}{\sqrt{\textrm{\textsc{kl}}(q_\phi(z) \;||\; p(z \g x))}}\right\vert }.
\end{equation}
\label{prop: Cov neglectable for large D}
\end{proposition}

\begin{proof}
See \Cref{proof: Cov neglectable for large D}.
\end{proof}

\begin{remark}
The variational approximation $q_\phi(z)$ typically underestimates the spread of the posterior $p(z \g x)$ \citep{blei2017variational}, and so the assumption in \Cref{eq:underspread} is typically satisfied in practice after a few iterations of the optimisation algorithm. The kurtosis $\mathrm{Kurt}[\partial_{\phi_i} \log q_\phi]$ quantifies the weight of the tails of the variational approximation in terms of the score function. In \Cref{app:exponential families} we analyse the kurtosis of exponential family distributions and show that it is uniformly bounded for Gaussian variational families.
\end{remark}

\begin{remark}\label{rem:largeKLremark}
The upper bound in \Cref{eq:cor bound} allows us to identify two regimes. When $\textrm{\textsc{kl}}(q_\phi(z) \;||\; p(z \g x))$ is large, the bound asserts that the relative error satisfies
\begin{align}\label{eq:regime1}
\left\vert \frac{\delta^{\textnormal{CV}}_i}{\mathbb{E}_{q_\phi}[a^{\text{VarGrad}}]}\right\vert \lessapprox \mathcal{O}\left(\textrm{\textsc{kl}}(q_\phi(z) \;||\; p(z \g x))^{-1/2}\right),
\end{align}
as the second term in the denominator of \Cref{eq:cor bound} becomes negligible.  This can happen in the early stages of the optimisation process, in which case we can conclude that $\delta^{\text{CV}}$ is expected to be small. Since the \gls{KL} divergence increases with the dimensionality of the latent variable $z$ (see \Cref{app:KL}), \Cref{eq:cor bound} also implies that the ratio becomes smaller as the number of latent variables grows. Moreover, if the \textit{minimum} \gls{KL} divergence between the variational family and the true posterior is large (i.e., if the best candidate in the variational family is still far away from the target), the correction term $\delta_i^{\text{CV}}$ can be negligible during the whole optimisation procedure, which is often the case in practice.

In the regime where $\textrm{\textsc{kl}}(q_\phi(z) \;||\; p(z \g x))$ approaches zero (i.e., towards the end of the optimisation process if the variational family is well specified and includes the posterior), then \Cref{eq:cor bound} implies that
\begin{align}
\label{eq:regime2}
\left\vert \frac{\delta^{\text{CV}}_i}{\mathbb{E}_{q_\phi}[a^{\text{VarGrad}}]}\right\vert \lessapprox  \mathcal{O}\left(\textrm{\textsc{kl}}(q_\phi(z) \;||\; p(z \g x))^{1/2}\right).
\end{align}
In this regime, the error w.r.t. the optimal control variate coefficient decreases with the $\gls{KL}$ divergence. The estimates in \Cref{eq:regime1} and \Cref{eq:regime2} combined suggest that the relative error remains bounded throughout the optimisation. We verify this proposition experimentally in \Cref{sec:experiments}.
\end{remark}

\subsection{Variance of the Estimator}
\label{sec:variance}

In this section we provide a result that guarantees that the variance of $\widehat{g}_{\text{VarGrad}}$ is smaller than the variance of $\widehat{g}_{\text{Reinforce}}$ when the number of Monte Carlo samples is large enough.

\begin{proposition}
\label{prop: variance comparison log-var vs. reinforce}
Consider the two gradient estimators $\widehat{g}_{\text{Reinforce}}(\phi)$ and $\widehat{g}_{\text{VarGrad}}(\phi)$, each with $S$ Monte Carlo samples, as defined in \Cref{eq:score_function} and \Cref{eq:vargrad}, respectively. If 
\begin{equation}
\label{eq:comparison bound}
-\frac{\delta^{\text{CV}}_i}{\mathbb{E}_{q_\phi}[a^{\text{VarGrad}}]} =  \frac{\delta^{\text{CV}}_i}{\mathrm{ELBO}(\phi)} < \frac{1}{2}
\end{equation}
then there exists $S_0 \in \mathbb{N}$ such that 
\begin{equation}
\label{eq:better variance}
\Var \left(\widehat{g}_{\text{VarGrad}, i}(\phi)\right) \leq \Var \left(\widehat{g}_{ \text{Reinforce}, i}(\phi)\right), \quad \quad \textnormal{for all} \quad S \ge S_0.
\end{equation}
\end{proposition}
\begin{proof}
See \Cref{proof: variance comparison log-var vs. reinforce}.
\end{proof}

If the correction $\delta^{\text{CV}}$ is negligible in the sense of \Cref{prop:delta small}, then the assumption in \Cref{eq:comparison bound} is satisfied and  \Cref{prop: variance comparison log-var vs. reinforce} guarantees that VarGrad has lower variance than Reinforce when $S$ is large enough.
We arrive at the following corollary, which also considers the dimensionality of the latent variables. The main assumption -that the \gls{KL}-divergence increases with the dimension of the latent space- is supported by the result in \Cref{app:KL}.

\begin{corollary}
\label{corollary: variance comparison log-var vs. reinforce}
Let $S$ be the number of samples and $D$ the dimension of the latent variable $z$. Furthermore, let the assumptions of \Cref{prop: Cov neglectable for large D} be satisfied and assume that $\textrm{\textsc{kl}}(q_\phi(z) \;||\; p(z \g x))$ is strictly  increasing in $D$. Then, there exist $S_0, D_0 \in \mathbb{N}$ such that
\begin{equation}
\Var \left(\widehat{g}_{\text{VarGrad}, i}(\phi)\right) \leq \Var \left(\widehat{g}_{\text{Reinforce}, i}(\phi)\right), \quad \quad \textnormal{for all } S \ge S_0 \textnormal{ and } D \ge D_0.
\end{equation}
\end{corollary}

\begin{proof}
See \Cref{proof: corollary variance comparison log-var vs. reinforce}.
\end{proof}

We provide further intuition on the condition in \Cref{eq:comparison bound} with the analysis in \Cref{appendix: Covariance terms for diagonal Gaussians}. 
\section{Related Work}
\label{sec:related}

In the last few years, many gradient estimators of the \gls{ELBO} have been proposed; see \citet{mohamed2019monte} for a comprehensive review. Among those, the score function estimators \citep{williams1992simple,carbonetto2009stochastic,paisley2012variational,ranganath2014black} and the reparameterisation estimators \citep{kingma2014auto,rezende2014stochastic,titsias2014doubly}, as well as combinations of both \citep{ruiz2016generalized,Naesseth2017}, are arguably the most widely used.
\textsc{nvil} \citep{mnih2014neural} and MuProp \citep{gu2016muprop} are unbiased gradient estimators for training stochastic neural networks.

Other gradient estimators are specific for discrete-valued latent variables. The concrete relaxation \citep{Maddison2017TheCD,jang2017} described a way to form a biased estimator of the gradient, which \textsc{rebar} \citep{tucker2017rebar} and \textsc{relax} \citep{grathwohl2018backpropagation} use as a control variate to obtain an unbiased estimator. Other recent estimators have been proposed by \citet{lee2018reparameterization,peters2018probabilistic,shayer2018learning,cong2019go,yin2018arm,yin2019arsm}, and \citet{dong2020disarm}.
In \Cref{sec:experiments}, we compare VarGrad with some of these estimators, showing that it exhibits a favourable performance versus computational complexity trade-off.

The VarGrad estimator was first introduced by \citet{salimans2014onusing} and \citet{kool2019buy}. It also relates to \textsc{vimco} \citep{mnih2016variational} in that it is a leave-one-out estimator. In this paper, we have described an alternative derivation of VarGrad, based on the log-variance loss.

The log-variance loss from \Cref{subsec:log_variance_loss} defines an alternative divergence between the approximate and the exact posterior distributions. In the context of optimal control of diffusion 
processes and related forward-backward stochastic differential equations, it arises naturally to quantify the discrepancy between measures on path space \citep{nusken2020solving}. 
Other forms of alternative divergences have also been explored in previous work; for example the $\chi^2$-divergence \citep{dieng2017variational}, the R\'{e}nyi divergence \citep{li2016renyi}, the Langevin-Stein \citep{ranganath2016operator}, the  $\alpha$-divergence \citep{hernandezlobato2016black}, other $f$-divergences \citep{wang2018variational}, a contrastive divergence \citep{ruiz2019contrastive}, and also the inclusive \gls{KL} \citep{naesseth2020markovian}, see also \Cref{appendix: other divergences}.

Finally, from an implementation perspective, \Cref{alg:vargrad} contains a \verb+stop_gradient+ operator that resembles the method of \citet{roeder2017sticking}. In \citet{roeder2017sticking}, this operator is used on the variational parameters to eliminate the entropy term in the reparameterisation gradient to reduce its variance; this has also been extended to importance weighted variational objectives \citep{tucker2018doubly}. In contrast, VarGrad applies the \verb+stop_gradient+ operator on the samples and is based on the score function method.

\section{Experiments}
\label{sec:experiments}

\begin{figure}[bt]
\begin{subfigure}[t]{.33\textwidth}
  \centering
  \includegraphics[width=.9\linewidth]{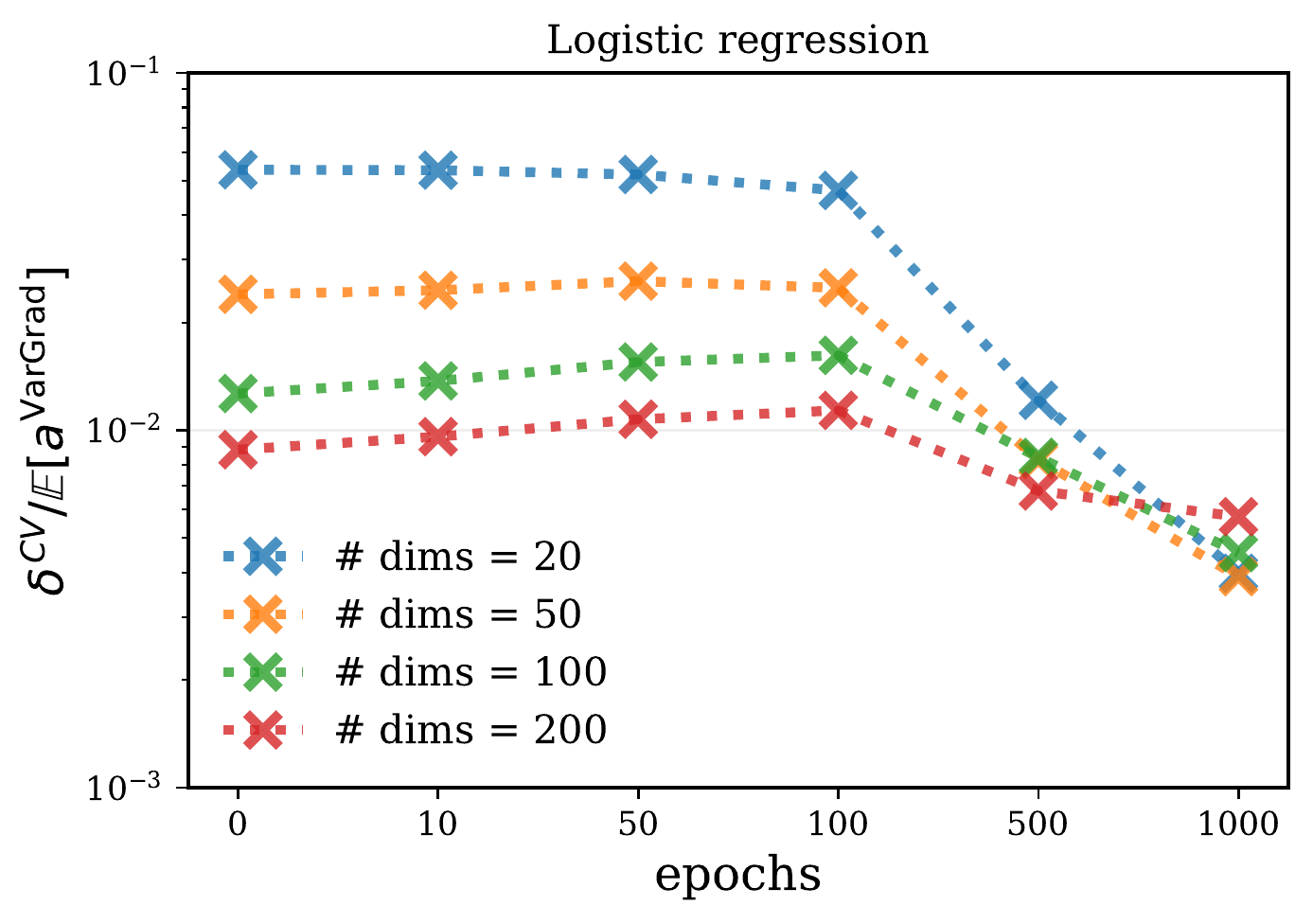}
    \caption{}
\end{subfigure}
\begin{subfigure}[t]{.33\textwidth}
  \centering
  \includegraphics[width=.9\linewidth]{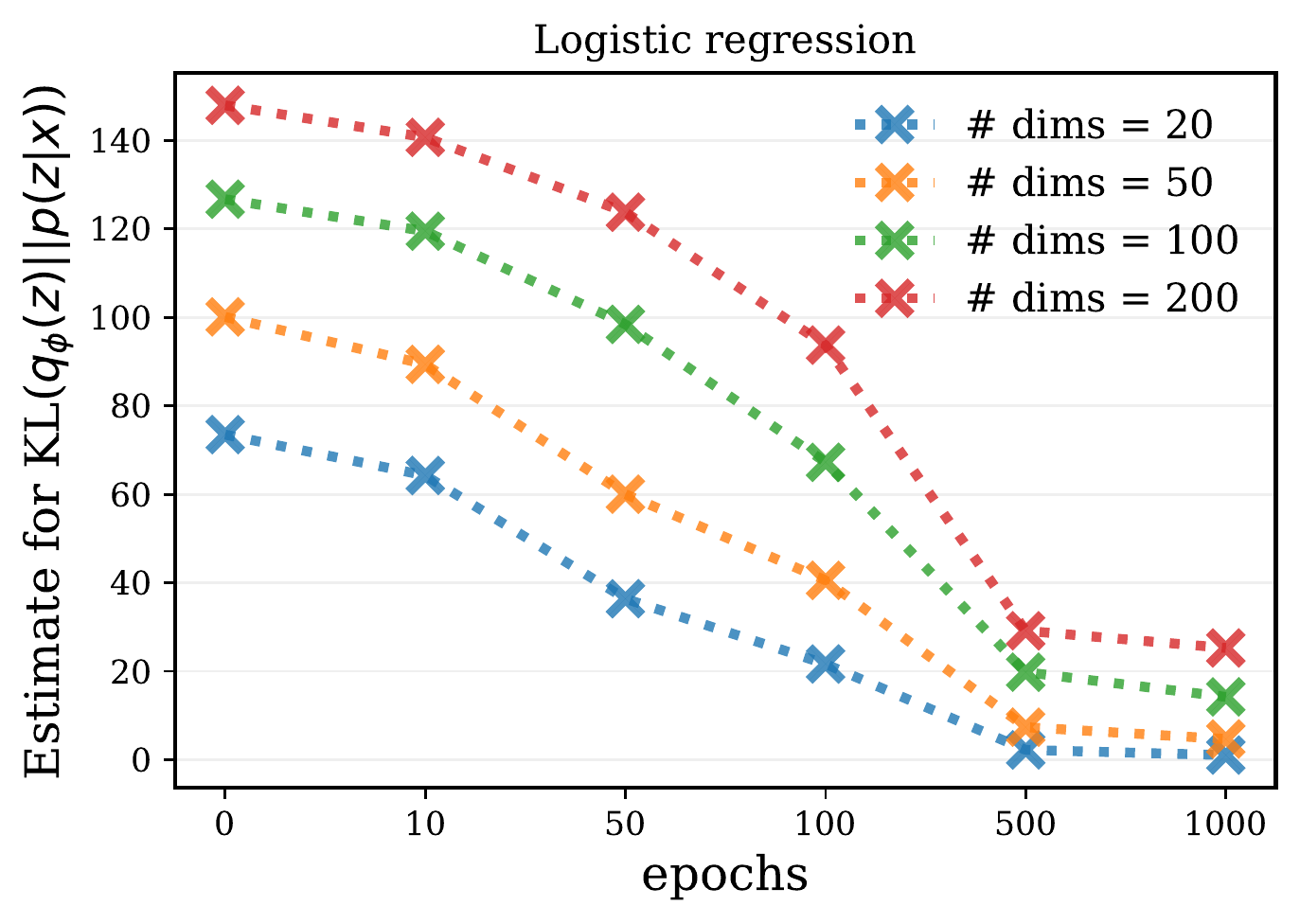}
    \caption{}
\end{subfigure}
\begin{subfigure}[t]{.33\textwidth}
  \centering
  \includegraphics[width=.9\linewidth]{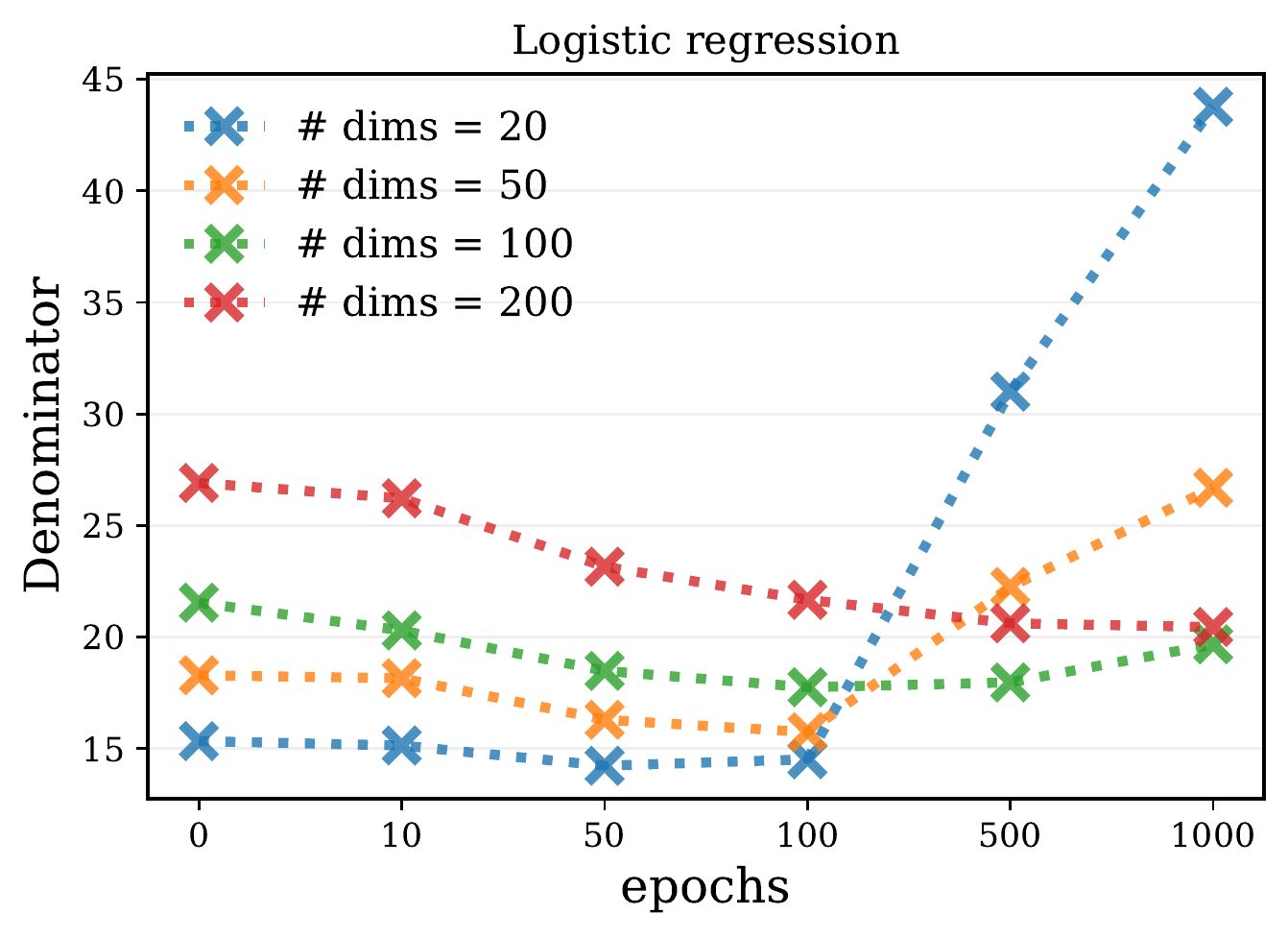}
    \caption{}
\end{subfigure}
    \caption{Illustration of \Cref{prop:delta small} and \Cref{rem:largeKLremark} on the logistic regression model. In (a), we show that the ratio $\left\vert {\delta^{\textnormal{CV}}_i}/{\mathbb{E}_{q_\phi}[a^{\text{VarGrad}}]}\right\vert$ is small and uniformly bounded over epochs, illustrating that the VarGrad estimator stays close to the optimal control variate coefficients during the whole optimisation procedure. Additionally, this ratio decreases with increasing dimensionality of the latent variables. In (b), we display an estimate of the \gls{KL} divergence across epochs and demonstrate the beneficial effect of higher dimensions, since the bound of \Cref{eq:cor bound} is expected to scale like $\mathcal{O}(\mathrm{\textsc{kl}}^{-1/2})$ in the early phase. In (c), we plot an estimate of the denominator of the bound (\Cref{eq:cor bound}), which increases or stays constant over epochs, demonstrating that the ratio in \Cref{eq:cor bound} stays stable (and small) over epochs.}
\label{fig:maindeltacv_log_reg}
\end{figure}

In order to verify the properties of VarGrad empirically, we test it on two popular models: a Bayesian logistic regression model on a synthetic dataset and a \gls{DVAE} \citep{salakhutdinov2008quantitative, kingma2014auto} on a fixed binarisation of Omniglot \citep{lake2015human}.
All details of the experiments can be found in \Cref{app:details_experiments}.\footnote{%
Code in JAX \citep{jax2018github,haiku2020github} is available at \url{https://github.com/aboustati/vargrad}.}

\parhead{Closeness to the optimal control variate.}
In \Cref{sec:analytical} we analytically showed that VarGrad is close to the optimal control variate, and in particular that the ratio $\left\vert {\delta^{\textnormal{CV}}_i}/{\mathbb{E}_{q_\phi}[a^{\text{VarGrad}}]}\right\vert$ can be small over the whole optimisation procedure. This behaviour is expected to be even more pronounced with growing dimensionality of the latent space. In \Cref{fig:maindeltacv_log_reg}, we confirm this result by showing the ratio $\left\vert {\delta^{\textnormal{CV}}_i}/{\mathbb{E}_{q_\phi}[a^{\text{VarGrad}}]}\right\vert$ for the logistic regression model. 
We also show the \gls{KL} divergence along the iterations and the denominator of the bound in \Cref{eq:cor bound}; see \Cref{fig:maindeltacv_log_reg} for the details.

In \Cref{fig:vae_delta_per_dim}, we provide further evidence that this ratio is also small when fitting \glspl{DVAE}. Indeed, we observe that the ratio $\delta_i^{\textrm{CV}}/\bE[a^{\textnormal{VarGrad}}]$ is typically very small and is distributed around zero during the whole optimisation procedure.

\begin{figure}[t]
    \centering
    \begin{subfigure}[t]{0.48\textwidth}
      \centering
      \includegraphics[width=\textwidth]{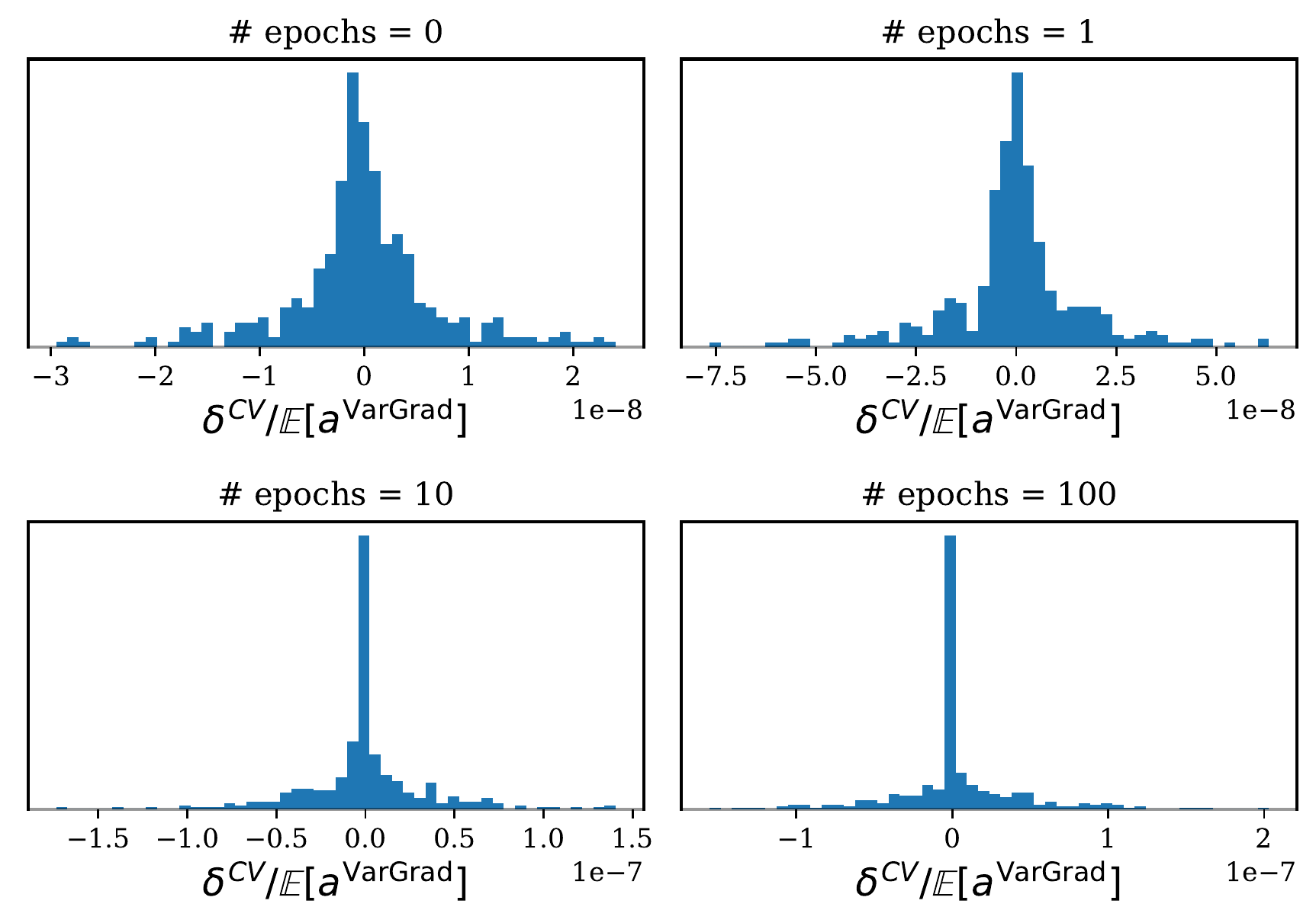}  
      \caption{Two-layer linear \acrshort{DVAE}.\label{fig:vae_delta_per_dim_linear}}
    \end{subfigure}
    \begin{subfigure}[t]{0.48\textwidth}
      \centering
      \includegraphics[width=\linewidth]{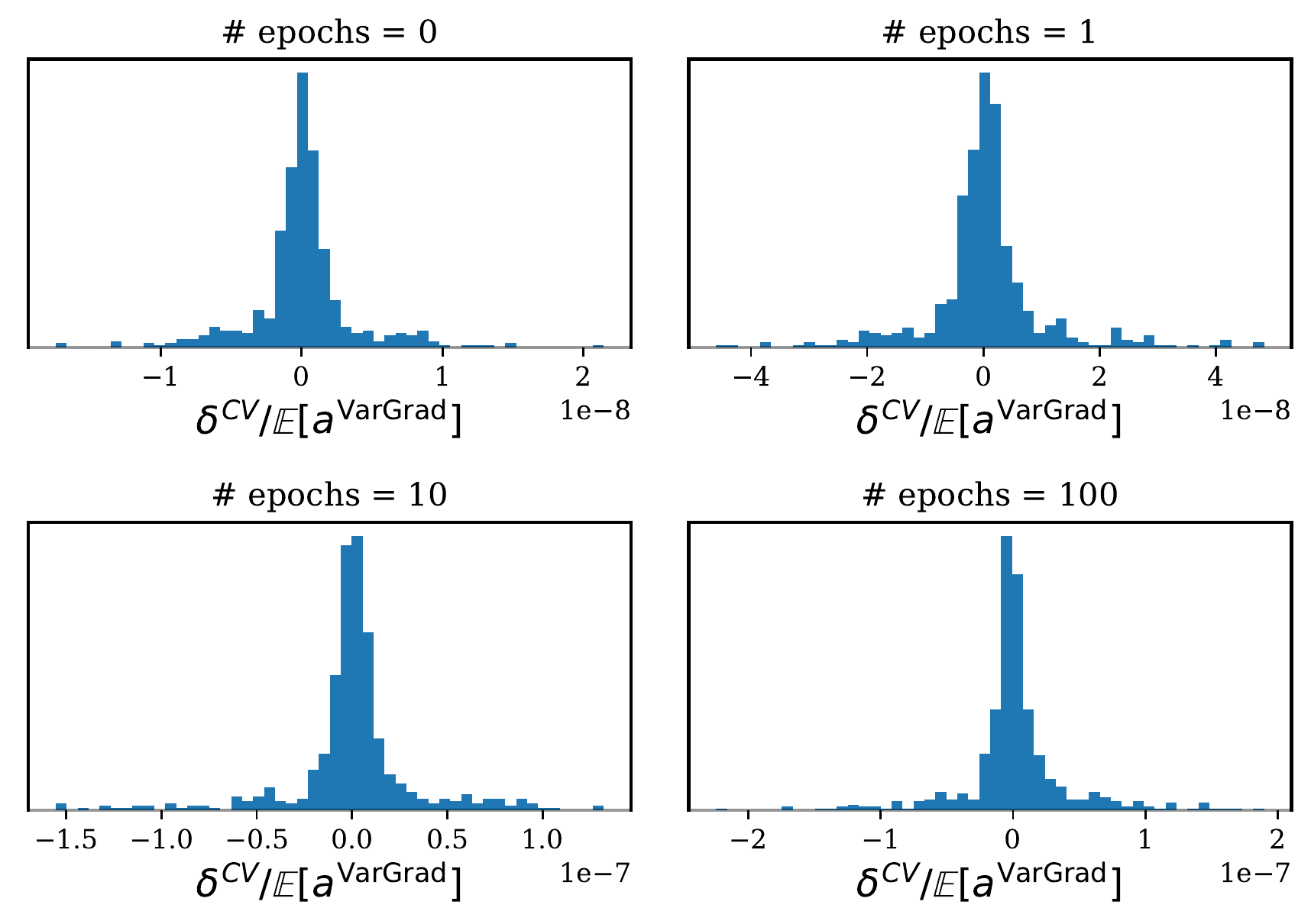}  
      \caption{Non-linear \acrshort{DVAE}.\label{fig:vae_delta_per_dim_deep}}
    \end{subfigure}
    \caption{The distribution of $\frac{\delta_i^{\text{CV}}}{\mathbb{E}[a^{\text{VarGrad}}]}$ associated with the biases of two \gls{DVAE} models with $200$ latent dimensions trained on Omniglot using VarGrad. The estimates are obtained with $2{,}000$ Monte Carlo samples. The ratio $\frac{\delta_i^{\text{CV}}}{\mathbb{E}[a^{\text{VarGrad}}]}$ is consistently small throughout the optimisation procedure.%
    \label{fig:vae_delta_per_dim}}
\end{figure}

\parhead{Variance reduction and computational cost.}
In \Cref{fig:logistic_regression_variance} we show the variance of different gradient estimators throughout the optimisation in the logistic regression setting. We realise a significant improvement of VarGrad compared to the standard Reinforce estimator (\Cref{eq:score_function}). In fact, we observe a small difference between the variance of VarGrad and the variance of an \emph{oracle estimator} based on Reinforce with access to the optimal control variate coefficient $a^*$. \Cref{fig:logistic_regression_variance} also shows the variance of the \emph{sampled estimator}, which is based on Reinforce with an estimate of the optimal control variate; this confirms the difficulty of estimating it in practice. (A similar trend can be observed for the \gls{DVAE} in the results in \Cref{app:details_experiments}, where VarGrad is compared to a wider list of estimators from the \gls{DVAE} literature.)
All methods use $S=4$ Monte Carlo samples, and the control variate coefficient is estimated with either $2$ extra samples (\emph{sampled estimator}) or $1{,}000$ samples (\emph{oracle estimator}).

\begin{figure}[t]
    \centering
    \includegraphics[width=0.41\textwidth]{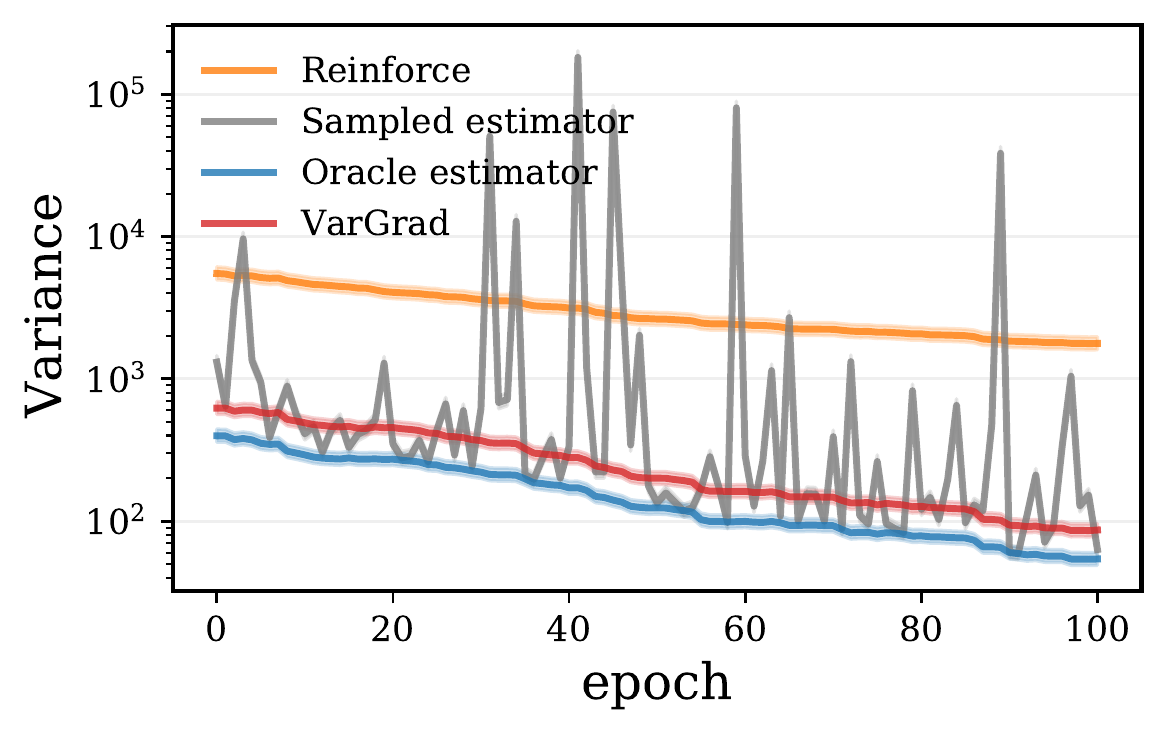}
    \caption{Estimates of the variance of the gradient component w.r.t.\ the posterior mean of one of the weights for the logistic regression model. The variance of VarGrad is close to the \emph{oracle estimator} based on Reinforce with access to the optimal control variate coefficient $a^*$. Moreover, the \emph{sampled estimator} (based on Reinforce with an estimate of $a^*$) shows the difficulty of estimating the optimal control variate coefficient in practice.}
    \label{fig:logistic_regression_variance}
\end{figure}

Finally, \Cref{fig:omniglot_main} compares VarGad with other estimators by training a \gls{DVAE} on Omniglot. The figure shows the negative \gls{ELBO} as a function of the epoch number (left plot) and against the wall-clock time (right plot). The negative \gls{ELBO} is computed on the standard test split and the optimisation uses Adam \citep{kingma2015adam} with learning rate of $0.001$. VarGrad achieves similar performance to state-of-the-art estimators, such as \textsc{rebar} \citep{tucker2017rebar}, \textsc{relax} \citep{grathwohl2018backpropagation}, and \textsc{arm} \citep{yin2018arm}, while being simpler to implement (see \Cref{alg:vargrad}) and without any tunable hyperparameters.

\begin{figure}[t]
\centering
\includegraphics[width=\linewidth]{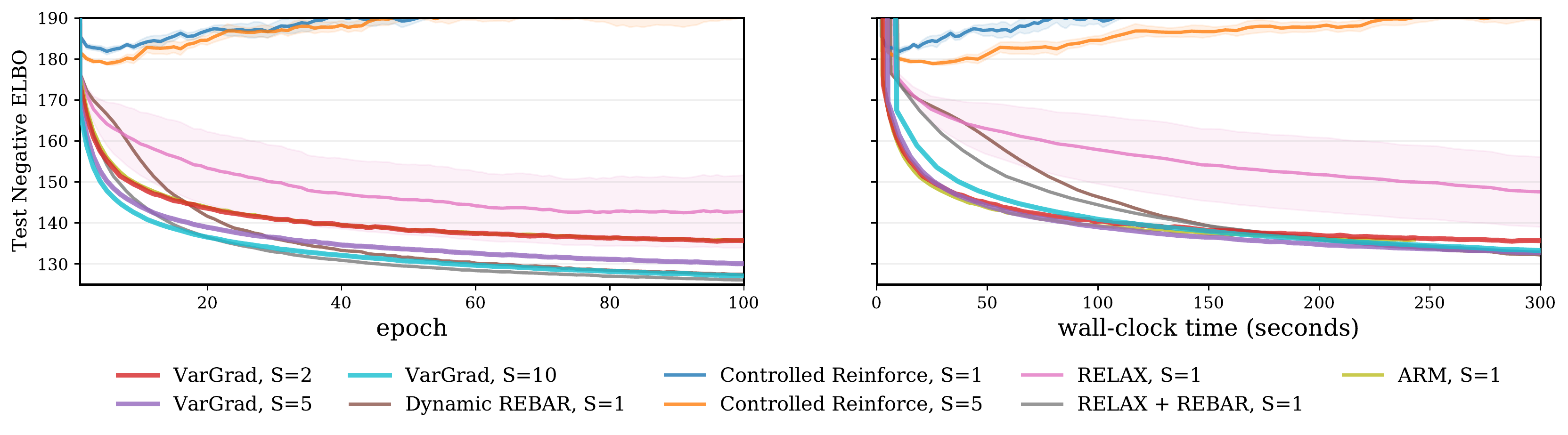}
\caption{Optimisation trace versus epoch (left) and wall-clock time (right) for a two-layer linear \gls{DVAE} on a fixed binarisation of Omniglot. The plot compares VarGrad to Reinforce with score function control variates \citep{ranganath2014black}, dynamic \textsc{rebar} \citep{tucker2017rebar}, \textsc{relax}, \textsc{relax} + \textsc{rebar} \citep{grathwohl2018backpropagation} and ARM \citep{yin2018arm}. The number of samples used to compute each gradient estimator is given in the figure legend. VarGrad demonstrates favourable scalability and performance when compared to the other estimators.}
\label{fig:omniglot_main}
\end{figure} 
\section{Conclusions}
\label{sec:conclusions}

We have analysed the VarGrad estimator, an estimator of the gradient of the \gls{KL} that is based on Reinforce with leave-one-out control variates, which was first introduced by \citet{salimans2014onusing} and \citet{kool2019buy}. We have established the connection between VarGrad and a novel divergence, which we call the log-variance loss. We have showed theoretically that, under certain conditions, the VarGrad control variate coefficients are close to the optimal ones. Moreover, we have established the conditions that guarantee that VarGrad exhibits lower variance than Reinforce. We leave it for future work to explore the direct optimisation of the log-variance loss for alternative choices of the reference distribution $r(z)$.

\section*{Acknowledgments}

\"{O}.~D.~A. and A.~B.\ are funded by the Lloyds Register Foundation programme on Data Centric Engineering through the London Air Quality project at The Alan Turing Institute for Data Science and AI. This work was supported under EPSRC grant EP/N510129/1 as well as by Deutsche Forschungsgemeinschaft (DFG) through the
grant CRC 1114 ‘Scaling Cascades in Complex Systems’ (projects A02 and A05, project number 235221301).

\section*{Broader Impact}
Variational inference algorithms are approximate inference methods used in many practical applications, including computational neuroscience, natural language processing, and computer vision; see \citet{blei2017variational}. The performance of variational inference often depends on the variance of the gradient estimator of its objective. High-variance estimators can make the resulting algorithm unstable and unreliable to use in real-world deployment scenarios; hence, deeper theoretical and empirical understanding of different gradient estimators is crucial for safe applicability of these methods in real-world settings.

In our work, we provide an analysis of the VarGrad estimator. We show the connection between this estimator and the ``log-variance loss'' and demonstrate its relationship to the optimal score function control variate for Reinforce, providing a theoretical analysis. We support our theoretical analysis with empirical results. We believe our work contributes in the further understanding of variational inference methods and contributes to improving their safety and applicability.

While our theoretical results are suggestive, we warn that the bounds depend on the nature of the model at hand. Therefore, the use of the results here while ignoring the model-specific aspects of our assumptions may lead to some risks in application.

\bibliography{manuscript}

\begin{thebibliography}{48}
\providecommand{\natexlab}[1]{#1}
\providecommand{\url}[1]{\texttt{#1}}
\expandafter\ifx\csname urlstyle\endcsname\relax
  \providecommand{\doi}[1]{doi: #1}\else
  \providecommand{\doi}{doi: \begingroup \urlstyle{rm}\Url}\fi

\bibitem[Blei et~al.(2017)Blei, Kucukelbir, and McAuliffe]{blei2017variational}
D.~M. Blei, A.~Kucukelbir, and J.~D. McAuliffe.
\newblock Variational inference: {A} review for statisticians.
\newblock \emph{Journal of the American Statistical Association}, 112\penalty0
  (518):\penalty0 859--877, 2017.

\bibitem[Bradbury et~al.(2018)Bradbury, Frostig, Hawkins, Johnson, Leary,
  Maclaurin, and Wanderman-Milne]{jax2018github}
J.~Bradbury, R.~Frostig, P.~Hawkins, M.~J. Johnson, C.~Leary, D.~Maclaurin, and
  S.~Wanderman-Milne.
\newblock {JAX}: composable transformations of {P}ython+{N}um{P}y programs,
  2018.
\newblock URL \url{http://github.com/google/jax}.

\bibitem[Carbonetto et~al.(2009)Carbonetto, King, and
  Hamze]{carbonetto2009stochastic}
P.~Carbonetto, M.~King, and F.~Hamze.
\newblock A stochastic approximation method for inference in probabilistic
  graphical models.
\newblock In \emph{Advances in Neural Information Processing Systems}, 2009.

\bibitem[Cong et~al.(2019)Cong, Zhao, Bai, and Carin]{cong2019go}
Y.~Cong, M.~Zhao, K.~Bai, and L.~Carin.
\newblock {GO} gradient for expectation-based objectives.
\newblock In \emph{International Conference on Learning Representations}, 2019.

\bibitem[Cover and Thomas(2012)]{cover2012elements}
T.~M. Cover and J.~A. Thomas.
\newblock \emph{Elements of {I}nformation {T}heory}.
\newblock John Wiley \& Sons, 2012.

\bibitem[Dieng et~al.(2017)Dieng, Tran, Ranganath, Paisley, and
  Blei]{dieng2017variational}
A.~B. Dieng, D.~Tran, R.~Ranganath, J.~Paisley, and D.~M. Blei.
\newblock Variational inference via $\chi$-upper bound minimization.
\newblock In \emph{Advances in Neural Information Processing Systems}, 2017.

\bibitem[Dong et~al.(2020)Dong, Mnih, and Tucker]{dong2020disarm}
Z.~Dong, A.~Mnih, and G.~Tucker.
\newblock {DisARM}: An antithetic gradient estimator for binary latent
  variables.
\newblock In \emph{Advances in Neural Information Processing Systems}, 2020.

\bibitem[Ghosal et~al.(2000)Ghosal, Ghosh, Van Der~Vaart,
  et~al.]{ghosal2000convergence}
S.~Ghosal, J.~K. Ghosh, A.~W. Van Der~Vaart, et~al.
\newblock Convergence rates of posterior distributions.
\newblock \emph{Annals of Statistics}, 28\penalty0 (2):\penalty0 500--531,
  2000.

\bibitem[Grathwohl et~al.(2018)Grathwohl, Choi, Wu, Roeder, and
  Duvenaud]{grathwohl2018backpropagation}
W.~Grathwohl, D.~Choi, Y.~Wu, G.~Roeder, and D.~Duvenaud.
\newblock Backpropagation through the void: Optimizing control variates for
  black-box gradient estimation.
\newblock In \emph{International Conference on Learning Representations}, 2018.

\bibitem[Gu et~al.(2016)Gu, Levine, Sutskever, and Mnih]{gu2016muprop}
S.~Gu, S.~Levine, I.~Sutskever, and A.~Mnih.
\newblock Mu{P}rop: Unbiased backpropagation for stochastic neural networks.
\newblock In \emph{International Conference on Machine Learning}, 2016.

\bibitem[Hennigan et~al.(2020)Hennigan, Cai, Norman, and
  Babuschkin]{haiku2020github}
T.~Hennigan, T.~Cai, T.~Norman, and I.~Babuschkin.
\newblock {H}aiku: {S}onnet for {JAX}, 2020.
\newblock URL \url{http://github.com/deepmind/dm-haiku}.

\bibitem[Hern\'{a}ndez-Lobato et~al.(2016)Hern\'{a}ndez-Lobato, Rowland,
  Hern\'{a}ndez-Lobato, Bui, and Turner]{hernandezlobato2016black}
Y.~Hern\'{a}ndez-Lobato, J. M. amd~Li, M.~Rowland, D.~Hern\'{a}ndez-Lobato,
  T.~Bui, and R.~E. Turner.
\newblock Black-box $\alpha$-divergence minimization.
\newblock In \emph{International Conference on Machine Learning}, 2016.

\bibitem[Jang et~al.(2017)Jang, Gu, and Poole]{jang2017}
E.~Jang, S.~Gu, and B.~Poole.
\newblock Categorical reparameterization with {G}umbel-softmax.
\newblock In \emph{International Conference on Learning Representations}, 2017.

\bibitem[Jordan et~al.(1999)Jordan, Ghahramani, Jaakkola, and
  Saul]{jordan1999introduction}
M.~I. Jordan, Z.~Ghahramani, T.~S. Jaakkola, and L.~K. Saul.
\newblock An introduction to variational methods for graphical models.
\newblock \emph{Machine Learning}, 37\penalty0 (2):\penalty0 183--233, Nov.
  1999.

\bibitem[Kingma and Ba(2015)]{kingma2015adam}
D.~P. Kingma and J.~Ba.
\newblock Adam: A method for stochastic optimization.
\newblock In \emph{International Conference on Learning Representations
  (ICLR)}, 2015.

\bibitem[Kingma and Welling(2014)]{kingma2014auto}
D.~P. Kingma and M.~Welling.
\newblock Auto-encoding variational {B}ayes.
\newblock In \emph{International Conference on Learning Representations}, 2014.

\bibitem[Kool et~al.(2019)Kool, van Hoof, and Welling]{kool2019buy}
W.~Kool, H.~van Hoof, and M.~Welling.
\newblock Buy 4 {REINFORCE} samples, get a baseline for free!
\newblock In \emph{{ICLR} Workshop on Deep Reinforcement Learning Meets
  Structured Prediction}, 2019.

\bibitem[Kool et~al.(2020)Kool, van Hoof, and Welling]{kool2020estimating}
W.~Kool, H.~van Hoof, and M.~Welling.
\newblock Estimating gradients for discrete random variables by sampling
  without replacement.
\newblock In \emph{International Conference on Learning Representations}, 2020.

\bibitem[Lake et~al.(2015)Lake, Salakhutdinov, and Tenenbaum]{lake2015human}
B.~M. Lake, R.~Salakhutdinov, and J.~B. Tenenbaum.
\newblock Human-level concept learning through probabilistic program induction.
\newblock \emph{Science}, 350\penalty0 (6266):\penalty0 1332--1338, 2015.

\bibitem[Lee et~al.(2018)Lee, Yu, and Yang]{lee2018reparameterization}
W.~Lee, H.~Yu, and H.~Yang.
\newblock Reparameterization gradient for non-differentiable models.
\newblock In \emph{Advances in Neural Information Processing Systems}, 2018.

\bibitem[Li and Turner(2016)]{li2016renyi}
Y.~Li and R.~E. Turner.
\newblock R\'{e}nyi divergence variational inference.
\newblock In \emph{Advances in Neural Information Processing Systems}, 2016.

\bibitem[Maddison et~al.(2017)Maddison, Mnih, and Teh]{Maddison2017TheCD}
C.~J. Maddison, A.~Mnih, and Y.~W. Teh.
\newblock The concrete distribution: A continuous relaxation of discrete random
  variables.
\newblock In \emph{International Conference on Learning Representations}, 2017.

\bibitem[Mnih and Gregor(2014)]{mnih2014neural}
A.~Mnih and K.~Gregor.
\newblock Neural variational inference and learning in belief networks.
\newblock In \emph{International Conference on Machine Learning}, 2014.

\bibitem[Mnih and Rezende(2016)]{mnih2016variational}
A.~Mnih and D.~J. Rezende.
\newblock Variational inference for {M}onte {C}arlo objectives.
\newblock In \emph{International Conference on Machine Learning}, 2016.

\bibitem[Mohamed et~al.(2019)Mohamed, Rosca, Figurnov, and
  Mnih]{mohamed2019monte}
S.~Mohamed, M.~Rosca, M.~Figurnov, and A.~Mnih.
\newblock Monte {C}arlo gradient estimation in machine learning.
\newblock \emph{arXiv preprint arXiv:1906.10652}, 2019.

\bibitem[Naesseth et~al.(2017)Naesseth, Ruiz, Linderman, and
  Blei]{Naesseth2017}
C.~Naesseth, F.~J.~R. Ruiz, S.~Linderman, and D.~M. Blei.
\newblock Reparameterization gradients through acceptance-rejection methods.
\newblock In \emph{Artificial Intelligence and Statistics}, 2017.

\bibitem[Naesseth et~al.(2020)Naesseth, Lindsten, and
  Blei]{naesseth2020markovian}
C.~A. Naesseth, F.~Lindsten, and D.~M. Blei.
\newblock Markovian score climbing: Variational inference with kl$(p||q)$.
\newblock \emph{arXiv preprint arXiv:2003.10374}, 2020.

\bibitem[N{\"u}sken and Richter(2020)]{nusken2020solving}
N.~N{\"u}sken and L.~Richter.
\newblock Solving high-dimensional {H}amilton-{J}acobi-{B}ellman {PDEs} using
  neural networks: perspectives from the theory of controlled diffusions and
  measures on path space.
\newblock \emph{arXiv preprint arXiv:2005.05409}, 2020.

\bibitem[Paisley et~al.(2012)Paisley, Blei, and Jordan]{paisley2012variational}
J.~W. Paisley, D.~M. Blei, and M.~I. Jordan.
\newblock Variational {B}ayesian inference with stochastic search.
\newblock In \emph{International Conference on Machine Learning}, 2012.

\bibitem[Peters and Welling(2018)]{peters2018probabilistic}
J.~W.~T. Peters and M.~Welling.
\newblock Probabilistic binary neural networks.
\newblock \emph{arXiv preprint arXiv:1809.03368}, 2018.

\bibitem[Ranganath et~al.(2014)Ranganath, Gerrish, and
  Blei]{ranganath2014black}
R.~Ranganath, S.~Gerrish, and D.~M. Blei.
\newblock Black box variational inference.
\newblock In \emph{Artificial Intelligence and Statistics}, 2014.

\bibitem[Ranganath et~al.(2016)Ranganath, Altosaar, Tran, and
  Blei]{ranganath2016operator}
R.~Ranganath, J.~Altosaar, D.~Tran, and D.~M. Blei.
\newblock Operator variational inference.
\newblock In \emph{Advances in Neural Information Processing Systems}, 2016.

\bibitem[Reiss(2012)]{reiss2012approximate}
R.-D. Reiss.
\newblock \emph{Approximate distributions of order statistics: with
  applications to nonparametric statistics}.
\newblock Springer science \& business media, 2012.

\bibitem[Rezende et~al.(2014)Rezende, Mohamed, and
  Wierstra]{rezende2014stochastic}
D.~J. Rezende, S.~Mohamed, and D.~Wierstra.
\newblock Stochastic backpropagation and approximate inference in deep
  generative models.
\newblock In \emph{International Conference on Machine Learning}, 2014.

\bibitem[Robbins and Monro(1951)]{RobbinsMonro}
H.~Robbins and S.~Monro.
\newblock A stochastic approximation method.
\newblock \emph{Annals of Mathematical Statistics}, 22:\penalty0 400--407,
  1951.

\bibitem[Roeder et~al.(2017)Roeder, Wu, and Duvenaud]{roeder2017sticking}
G.~Roeder, Y.~Wu, and D.~Duvenaud.
\newblock Sticking the landing: Simple, lower-variance gradient estimators for
  variational inference.
\newblock In \emph{Advances in Neural Information Processing Systems}, 2017.

\bibitem[Ruiz and Titsias(2019)]{ruiz2019contrastive}
F.~J.~R. Ruiz and M.~K. Titsias.
\newblock A contrastive divergence for combining variational inference and
  {MCMC}.
\newblock In \emph{International Conference on Machine Learning}, 2019.

\bibitem[Ruiz et~al.(2016)Ruiz, Titsias, and Blei]{ruiz2016generalized}
F.~J.~R. Ruiz, M.~K. Titsias, and D.~M. Blei.
\newblock The generalized reparameterization gradient.
\newblock In \emph{Advances in Neural Information Processing Systems}, 2016.

\bibitem[Salakhutdinov and Murray(2008)]{salakhutdinov2008quantitative}
R.~Salakhutdinov and I.~Murray.
\newblock On the quantitative analysis of deep belief networks.
\newblock In \emph{Proceedings of the 25th international conference on Machine
  learning}, pages 872--879, 2008.

\bibitem[Salimans and Knowles(2014)]{salimans2014onusing}
T.~Salimans and D.~A. Knowles.
\newblock On using control variates with stochastic approximation for
  variational bayes and its connection to stochastic linear regression.
\newblock \emph{arXiv preprint arXiv:1401.1022}, 2014.

\bibitem[Shayer et~al.(2018)Shayer, Levi, and Fetaya]{shayer2018learning}
O.~Shayer, D.~Levi, and E.~Fetaya.
\newblock Learning discrete weights using the local reparameterization trick.
\newblock In \emph{International Conference on Learning Representations}, 2018.

\bibitem[Titsias and L\'{a}zaro-Gredilla(2014)]{titsias2014doubly}
M.~K. Titsias and M.~L\'{a}zaro-Gredilla.
\newblock Doubly stochastic variational {B}ayes for non-conjugate inference.
\newblock In \emph{International Conference on Machine Learning}, 2014.

\bibitem[Tucker et~al.(2017)Tucker, Mnih, Maddison, and
  Sohl-Dickstein]{tucker2017rebar}
G.~Tucker, A.~Mnih, C.~J. Maddison, and J.~Sohl-Dickstein.
\newblock {REBAR:} low-variance, unbiased gradient estimates for discrete
  latent variable models.
\newblock In \emph{International Conference on Learning Representations}, 2017.

\bibitem[Tucker et~al.(2018)Tucker, Lawson, Gu, and Maddison]{tucker2018doubly}
G.~Tucker, D.~Lawson, S.~Gu, and C.~J. Maddison.
\newblock Doubly reparameterized gradient estimators for {M}onte {C}arlo
  objectives.
\newblock In \emph{International Conference on Learning Representations}, 2018.

\bibitem[Wang et~al.(2018)Wang, Liu, and Liu]{wang2018variational}
D.~Wang, H.~Liu, and Q.~Liu.
\newblock Variational inference with tail-adaptive $f$-divergence.
\newblock In \emph{Advances in Neural Information Processing Systems}, 2018.

\bibitem[Williams(1992)]{williams1992simple}
R.~J. Williams.
\newblock Simple statistical gradient-following algorithms for connectionist
  reinforcement learning.
\newblock \emph{Machine Learning}, 8\penalty0 (3--4):\penalty0 229--256, 1992.

\bibitem[Yin and Zhou(2019)]{yin2018arm}
M.~Yin and M.~Zhou.
\newblock {ARM}: Augment-{REINFORCE}-merge gradient for stochastic binary
  networks.
\newblock In \emph{International Conference on Learning Representations}, 2019.

\bibitem[Yin et~al.(2019)Yin, Yue, and Zhou]{yin2019arsm}
M.~Yin, Y.~Yue, and M.~Zhou.
\newblock {ARSM}: Augment-{REINFORCE}-swap-merge estimator for gradient
  backpropagation through categorical variables.
\newblock In \emph{International Conference on Machine Learning}, 2019.

\end{thebibliography}
\bibliographystyle{abbrvnat}

\clearpage
\appendix

\begin{center}
    \begin{huge}
    Supplementary Material
    \end{huge}
\end{center}

\renewcommand\thefigure{\thesection.\arabic{figure}}

\section{Proof of Paper Results}

\subsection{Proof of Proposition~\ref{proposition: equivalence log-variance relative entropy}}\label{proof:PropEquivalence}
\begin{proof}We first consider the gradient of the $\textrm{\textsc{kl}}$ divergence. It is given by
\begin{equation}
\label{eq:KL gradient}
    \nabla_\phi \textrm{\textsc{kl}}(q_\phi(z) \;||\; p(z \g x)) = \int \nabla_\phi q_\phi(z) \, \mathrm{d}z + \int \log\left( \frac{q_\phi(z)}{p(z \g x)} \right) \nabla_\phi q_\phi(z) \, \mathrm{d}z,
\end{equation}
where we can drop the first term since
$\int \nabla_\phi q_\phi(z)\, \mathrm{d}z = \nabla_\phi \int q_\phi(z)\, \mathrm{d}z = \nabla_\phi(1) = 0$.

We now consider the gradient of the log-variance loss. Using the definition from \Cref{eq:log_variance_loss}, we see that
\begin{align}
	& \nabla_{\phi} \mathcal{L}_r(q_\phi(z) \;||\; p(z \g x)) = \frac{1}{2} \nabla_\phi \int \log^2\left(\frac{q_\phi(z)}{p(z \g x)}\right) r(z) \,\md z - \frac{1}{2} \nabla_\phi \left(\int \log\left(\frac{q_\phi(z)}{p(z \g x)}\right) r(z) \,\md z \right)^2 \nonumber \\
	& = \int \log\left( \frac{q_\phi(z)}{p(z \g x)} \right) \frac{\nabla_\phi q_\phi(z)}{q_\phi(z)} r(z) \,\md z
	- \left( \int \log\left( \frac{q_\phi(z)}{p(z \g x)} \right) r(z) \,\md z \right) \left(\int \frac{\nabla_\phi q_\phi(z)}{q_\phi(z)}r(z) \,\md z \right).
	\nonumber
\end{align}
When we evaluate the gradient at $r(z) = q_\phi(z)$, the right-most term vanishes, since $\int \frac{\nabla_\phi q_\phi(z)}{r(z)}r(z) \,\md z = \int \nabla_\phi q_\phi(z) \,\md z=0$. Thus, the gradient of the log-variance loss becomes equal to the gradient of the $\textrm{\textsc{kl}}$ divergence.
\end{proof}

\subsection{Proof of Lemma \ref{lemma: optimal control variate decomposition}}
\label{proof: optimal control variate decomposition}

\begin{proof}
First, notice that $\Var_{q_\phi} \left( \partial_{\phi_i} \log q_\phi \right) = \mathbb{E}_{q_\phi}[(\partial_{\phi_i}\log q_\phi)^2]$ since $\mathbb{E}_{q_\phi}[\partial_{\phi_i}\log q_\phi] = 0$. We then compute
\begin{align}
a_i^* & =  \frac{\mathbb{E}_{q_\phi} \left[f_{\phi}  (\partial_{\phi_i} \log q_\phi)^2 \right]}{\mathbb{E}_{q_\phi} \left[\left( \partial_{\phi_i} \log q_\phi \right)^2 \right]}
\\
& = \frac{\mathbb{E}_{q_\phi} \left[f_{\phi}  (\partial_{\phi_i} \log q_\phi)^2 \right] - \mathbb{E}_{q_\phi}\left[f_\phi \right] \mathbb{E}_{q_\phi}\left[(\partial_{\phi_i} \log q_\phi)^2 \right]  + \mathbb{E}_{q_\phi}\left[f_\phi \right] \mathbb{E}_{q_\phi}\left[(\partial_{\phi_i} \log q_\phi)^2 \right]} {\mathbb{E}_{q_\phi} \left[\left( \partial_{\phi_i} \log q_\phi \right)^2 \right]}
\\
& = \mathbb{E}_{q_\phi}[\bar{f}_\phi] + \delta^{\text{CV}}_i.
\end{align}
In the last line we have used the fact that $\mathbb{E}_{q_\phi}[f_\phi] = \mathbb{E}_{q_\phi}[\bar{f}_\phi]$.
\end{proof}

\subsection{Proof of Proposition \ref{prop: Cov neglectable for large D}}
\label{proof: Cov neglectable for large D}
\begin{proof}
Note that
\begin{align}
\label{eq:rel error}
\left\vert  \frac{\delta^{\text{CV}}_i}{\mathbb{E}_{q_\phi}[a^{\text{VarGrad}}]}\right\vert = \left\vert  \frac{\Cov_{q_\phi}\left( f_\phi,\left(\partial_{\phi_i} \log q_\phi\right)^2\right)}{\mathbb{E}_{q_\phi}[f_\phi]\Var_{q_\phi} \left( \partial_{\phi_i} \log q_\phi \right)} \right \vert = \left\vert  \frac{\mathbb{E}_{q_\phi}\left[ (f_\phi - \mathbb{E}_{q_\phi}[f_\phi])\left(\partial_{\phi_i} \log q_\phi\right)^2\right]}{\mathbb{E}_{q_\phi}[f_\phi]\mathbb{E}_{q_\phi} \left[ (\partial_{\phi_i} \log q_\phi)^2 \right]} \right \vert,
\end{align}
where we have used the fact that $\mathbb{E}_{q_\phi}[\partial_{\phi_i} \log q_\phi] = 0$. From $$\mathbb{E}[f_\phi] = -\mathrm{ELBO}(\phi) = \mathrm{KL}(q_\phi(z) \;||\; p(z|x)) - \log p(x),$$
and using the Cauchy-Schwarz inequality, \Cref{eq:rel error} can be bounded from above by
\begin{equation}
\frac{\left(\Var_{q_\phi}\left( \log \frac{q_\phi(z)}{p(z|x)}\right)\right)^{1/2}}{|\mathrm{KL}(q_\phi(z) \;||\; p(z|x)) - \log p(x)|}  \left(\frac{\mathbb{E}_{q_\phi} [(\partial_{\phi_i} \log q_\phi)^4]}{(\mathbb{E}_{q_\phi} [(\partial_{\phi_i} \log q_\phi)^2])^2}\right)^{1/2}.
\end{equation}
The second factor equals $\sqrt{\mathrm{Kurt}[\partial_{\phi_i} \log q_\phi}]$.
To bound the first factor, notice that
\begin{align}
& \left(\Var_{q_\phi}\left( \log \frac{q_\phi(z)}{p(z|x)}\right)\right)^{1/2} \le \left(\mathbb{E}_{q_\phi}\left[ \log^2 \frac{q_\phi(z)}{p(z|x)} \right]\right)^{1/2} =  \left(\mathbb{E}_{q_\phi}\left[ \log^2 \frac{p(z|x)}{q_\phi(z)} \right]\right)^{1/2}
\\
\label{eq:exponential moment}
&\le \left(2 \mathbb{E}_{q_\phi} \left[ \exp \left( \left| \log \frac{p(z|x)}{q_\phi(z)}\right|\right) - 1 - \left\vert \log \frac{p(z|x)}{q_\phi(z)}\right\vert \right]\right)^{1/2},
\end{align}
where we have used the estimate
\begin{equation}
e^x - 1 - x = \sum_{n=0}^\infty \frac{x^n}{n!} - 1 - x = \sum_{n=2}^\infty \frac{x^n}{n!} \ge \frac{1}{2} x^2, \quad x \ge 0, 
\end{equation}
with $x = \left\vert \log \frac{p(z|x)}{q_\phi(z)}\right\vert$. We now use \citep[Lemma 8.3]{ghosal2000convergence} to bound \Cref{eq:exponential moment} from above by
\begin{equation}
2 \sqrt{C} h(q_\phi(z) \; || \; p(z|x)),
\end{equation}
where 
\begin{equation}
h(q_\phi(z) \; || \; p(z|x)) = \sqrt{\int \left( \sqrt{q_\phi(z)} - \sqrt{p(z|x)}\right)^2 \, \mathrm{d}z}
\end{equation}
is the Hellinger distance. From \citep[Lemma A.3.5]{reiss2012approximate} we have the bound $h(q_\phi(z) \; || \; p(z|x)) \le \sqrt{\mathrm{KL}(q_\phi(z) \; || \; p(z|x))}$. Combining these estimates we arrive at the claimed result.
\end{proof}

\subsection{Proof of Proposition \ref{prop: variance comparison log-var vs. reinforce}}
\label{proof: variance comparison log-var vs. reinforce}

\begin{proof}
We start by defining the short-cuts
\begin{equation}
A = f_\phi(z), \qquad B = \left(\partial_{\phi_i} \log q_\phi \right)(z).
\end{equation}
Let us compute the difference of the variances of the estimators to leading order in $S$, namely
\begin{subequations}
\begin{align}
& \Var(\widehat{g}_{\text{Reinforce}, i}) - \Var(\widehat{g}_{\text{VarGrad}, i}) = \frac{1}{S} \Var(AB) + \frac{S-2}{S(S-1)} \E\left[(A - \E[A])(B - \E[B]) \right]^2 \\
&\qquad- \frac{\Var(A)\Var(B)}{S(S-1)}  - \frac{1}{S}\E\left[(A - \E[A])^2(B - \E[B])^2 \right] \\
&= \frac{1}{S}\left( \E\left[A^2B^2 \right] - \E[AB]^2\right) + \frac{S-2}{S(S-1)}\E[AB]^2 \\
&\qquad -\frac{1}{S}\left(\E\left[A^2B^2 \right] - 2\E[A]\E\left[AB^2\right] + \E[A]^2\E\left[B^2\right] \right) + \mathcal{O}\left(\frac{1}{S^2}\right) \\
&= -\frac{1}{S(S-1)}\E[AB]^2 - \frac{1}{S}\E[A]\left(\E[A]\E\left[B^2\right] - 2 \E\left[ AB^2\right] \right) + \mathcal{O}\left(\frac{1}{S^2}\right)\\
&= \frac{1}{S}\E[A]\left(2 \E\left[ AB^2\right] - \E[A]\E\left[B^2\right] \right) + \mathcal{O}\left(\frac{1}{S^2}\right) \\
&= \frac{1}{S}\E[A]\E\left[B^2\right]\left(2 \delta_i^{\text{CV}} + \E[A] \right) + \mathcal{O}\left(\frac{1}{S^2}\right)
\end{align}
and we note that with $\E\left[B^2\right] > 0$ the leading term is positive if
\begin{align}
\E[A] \delta_i^{\text{CV}} + \frac{1}{2} \E[A]^2 > 0,
\end{align}
\end{subequations}
which is equivalent to the statement in the proposition.
\end{proof}

\subsection{Proof of Corollary \ref{corollary: variance comparison log-var vs. reinforce}}
\label{proof: corollary variance comparison log-var vs. reinforce}
\begin{proof}
Note that with Proposition \ref{prop: Cov neglectable for large D} we have
\begin{align}
     \left|\frac{\delta^{\text{CV}}_i}{\mathbb{E}_{q_\phi}[a^{\text{VarGrad}}]}\right| \to 0
\end{align}
for $D \to \infty$, assuming that $\textrm{\textsc{kl}}(q_\phi(z) \;||\; p(z \g x))$ is strictly increasing in $D$. Therefore, for large enough $D$, the condition from \Cref{prop: variance comparison log-var vs. reinforce} (see \Cref{eq:comparison bound}),
is fulfilled and the statement follows immediately. 
\end{proof}

\subsection{Results on the Kurtosis of the Score for Exponential Families}
\label{app:exponential families}
Here we provide a more explicit expression for $\mathrm{Kurt}[\partial_{\phi_i} \log q_\phi]$ in the case when $q_\phi(z)$ is given by an exponential family, i.e. $q_\phi(z) = h(z) \exp\left(\phi^\top T(z) - A(\phi)\right)$, where $T(z)$ is the vector of sufficient statistics and $A(\phi)$ denotes the log-partition function. As an application, we show that in the Gaussian case, $\mathrm{Kurt}[\partial_{\phi_i} \log q_\phi]$ is uniformly bounded across the whole variational family.
\begin{lemma}
Let $q_\phi(z) = h(z) \exp\left(\phi^\top T(z) - A(\phi)\right)$. Then
\begin{equation}
\mathrm{Kurt}[\partial_{\phi_i} \log q_\phi] =   \frac{\bE_{q_\phi}\left[(T_i(z) - m_i)^4\right]}{\bE_{q_\phi}\left[(T_i(z) - m_i)^2\right]^2},
\end{equation}
where $m_i = \bE_{q_\phi}[T_i(z)]$ denotes the mean of the sufficient statistics. In particular, $\mathrm{Kurt}[\partial_{\phi_i} \log q_\phi]$ does not depend on $h(z)$ or $A(\phi)$. 
\end{lemma}
\begin{proof}
The claim follows by direct calculation. Indeed,
\begin{align}
\partial_{\phi_i} \log q_\phi(z) = T_i(z) - \frac{\partial A}{\partial \phi_i}(\phi).
\end{align}
It is left to show that $\frac{\partial A}{\partial \phi_i}(\phi) = \mu_i$. For this, notice that the normalisation condition
\begin{equation}
\int h(z) \exp\left(\phi^\top T(z) - A(\phi)\right)\mathrm{d}z = 1
\end{equation}
implies 
\begin{equation}
\int h(z)\left( T_i(z) - \frac{\partial A(\phi)}{\partial \phi_i}\right) \exp\left(\phi^\top T(z) - A(\phi)\right)\mathrm{d}z = 0
\end{equation}
by taking the derivative w.r.t. $\phi_i$. The left-hand side equals $\mathbb{E}_{q_\phi}[T_i(z)] - \frac{\partial A}{\partial \phi_i}(\phi)$, and so the claim follows.
\end{proof}

\begin{lemma}\label{app:lem:GaussiansKurtosis}
Let $q_\phi(z)$ be the family of one-dimensional Gaussian distributions. Then there exists a constant $K>0$ such that
\begin{equation}
\mathrm{Kurt}[\partial_{\phi_i} \log q_\phi] < K
\end{equation}
for all $i$ and all $\phi \in \Phi$. In fact, it is possible to take $K=15$.
\end{lemma}
\begin{proof}
For the Gaussian family, the sufficient statistics are given by $T_1(z) = z$ and $T_2(z) = z^2$.
We have that 
\begin{equation}
\frac{\bE_{\mathcal{N}(\mu,\sigma^2)}\left[(T_1(z) - m_1)^4\right]}{\bE_{\mathcal{N}(\mu,\sigma^2)}\left[(T_1(z) - m_1)^2\right]^2} = 
\frac{\bE_{\mathcal{N}(\mu,\sigma^2)}\left[(z - \mu)^4\right]}{\bE_{\mathcal{N}(\mu,\sigma^2)}\left[(z - \mu)^2\right]^2} = 3,
\end{equation}
by the well-known fact the standard kurtosis of any univariate Gaussian is $3$. A lengthy but straightforward computation shows that 
\begin{equation}
\frac{\bE_{\mathcal{N}(\mu,\sigma^2)}\left[(T_2(z) - m_2)^4\right]}{\bE_{\mathcal{N}(\mu,\sigma^2)}\left[(T_2(z) - m_2)^2\right]^2} = \frac{3(4 \mu^4 + 20 \mu^2 \sigma^2 + 5 \sigma^4)}{(2 \mu^2 + \sigma^2)^2},
\end{equation}
which is maximised for $\mu = 0$, taking the value $15$.
\end{proof}
Lemma~\ref{app:lem:GaussiansKurtosis} shows that the kurtosis term in our bound \Cref{eq:cor bound} can be bounded for Gaussian families. This result is expected to extend to the multivariate cases as well. We note that we observe in our experiments that the bound is finite in a variety of cases.
\subsection{Dimension-dependence of the $\mathrm{KL}$-divergence}
\label{app:KL}

The following lemma shows that the $\mathrm{KL}$-divergence increases with the number of dimensions.  This result follows from the chain-rule of KL divergence, see, e.g., \citet{cover2012elements}.

\begin{lemma}
\label{lem:KL}
Let $u^{(D)}(z_1,\ldots,z_D)$ and $v^{(D)}(z_1,\ldots,z_D)$ be two arbitrary probability distributions on $\mathbb{R}^D$. For $J \in  \{1 \ldots,D \}$ denote their marginals on the first $J$ coordinates by $u^{(J)}$ and $v^{(J)}$, i.e. 
\begin{equation}
u^{(J)}(z_1,\ldots,z_J) = \idotsint u^{(D)}(z_1, \ldots, z_D) \, \mathrm{d}z_{J+1} \ldots \mathrm{d}z_D, 
\end{equation}
and 
\begin{equation}
v^{(J)}(z_1,\ldots,z_J) = \idotsint v^{(D)}(z_1, \ldots, z_D) \, \mathrm{d}z_{J+1} \ldots \mathrm{d}z_D.
\end{equation}
Then \begin{equation}
\mathrm{KL}(u^{(1)} \; || \; v^{(1)}) \le \mathrm{KL}(u^{(2)} \; || \; v^{(2)}) \le \ldots \le \mathrm{KL}(u^{(D)} \; || \; v^{(D)}),
\end{equation}
i.e. the function $J \mapsto \mathrm{KL}(u^{(J)} \; || \; v^{(J)})$ is increasing.
\end{lemma}

\newpage
\section{Details of Experimental Setup and Additional Results}
\label{app:details_experiments}
\subsection{Details of the Experiments}
\subsubsection{Logistic Regression}
This section describes the experimental setup of the Bayesian logistic regression example which was discussed in the main text in Section \ref{sec:experiments}.

\parhead{Data.} We use a synthetic dataset with $N=100$, where input-output pairs are generated as follows: we sample a design matrix $X \in \mathbb{R}^{N \times D}$ for the inputs uniformly on $[-1, 1]$, random weights $\w \in \mathbb{R}^D$ from $\mathcal{N}(0, 25 \,\text{Id}_{D \times D})$ and a random bias $b \in \mathbb{R}$ from $\mathcal{N}(0, 1)$. We set $\mathbf{p} = \sigma(X \w + b \, \mathbf{1})$, where $\mathbf{1}$ is an $N$-dimensional vector of ones and $\sigma(x) = \frac{1}{1 + \exp(-x)}$ is the logistic sigmoid applied elementwise. Finally, we sample the outputs $Y \sim \text{Bernoulli}(\mathbf{p})$.

\parhead{Approximate Posterior.} For this model, we set the approximate posterior to a diagonal Gaussian with free mean and log standard deviation parameters.

\parhead{Training.} For all the experiments listed in the main text, we use the VarGrad estimator for the gradients of the logistic regression models. We train the models using stochastic gradient descent \citep{RobbinsMonro} with a learning rate of $0.001$.

\parhead{Estimation of intractable quantities.} To estimate the intractable quantities in \Cref{fig:maindeltacv_log_reg}, we use Monte Carlo sampling with 2000 samples for $\delta^{\text{CV}}$ and $\mathbb{E}_{q_\phi}[a^{\text{VarGrad}}]$. We estimate the KL divergence with the identity $\textsc{KL}(q_\phi(z) \| p(z | x)) = \log p(x) - \text{ELBO}(\phi)$, where $\log p(x)$ is estimated using importance sampling with $10000$ samples and $\text{ELBO}(\phi)$ using standard Monte Carlo sampling with 2000 samples. 

For the variance estimates in \Cref{fig:logistic_regression_variance}, we use 1000 Monte Carlo samples. As explained in the main text, to estimate the control variate coefficients, we use either 2 samples for the \emph{sampled estimator} or 1000 samples for the \emph{oracle estimator}.

\subsection{Discrete VAEs}
This section describes the experimental setting for the Discrete VAE, where we closely follow the setup in \citet{Maddison2017TheCD}, which was also replicated in \citet{tucker2017rebar} and \citet{grathwohl2018backpropagation}. As we are comparing the usefulness of different estimators in the optimisation and time their run-times, we opted to re-implement the various methods using JAX \citep{jax2018github,haiku2020github}. Extra care was taken to be as faithful as possible to the implementation description in the respective papers as well as in optimising the run-time of the implementations.

\parhead{Data.} We use a fixed binarisation of Omniglot \citep{lake2015human}, where we binarise at the standard cut-off of 0.5. We use the standard train/test splits for this dataset.

\parhead{Model Architectures.} For the DVAE experiments we use the \emph{two layers linear} architecture, which has 2 stochastic binary layers with 200 units each, which was used in \citet{Maddison2017TheCD}. For this model, the decoders mirror the corresponding encoders. We use a Bernoulli(0.5) prior on the latent space and fix its parameters throughout the optimisation.

\parhead{Approximate Posterior.} We use an amortised mean-field Bernoulli approximation for the posterior.

\parhead{Training.} For training the models, we use the Adam optimiser \citep{kingma2015adam} with learning rates 0.001, 0.0005 and 0.0001.

\parhead{Estimation of intractable quantities}. We use Monte Carlo sampling with 2000 samples for $\delta^{\text{CV}}$ and $\mathbb{E}_{q_\phi}[a^{\text{VarGrad}}]$. Due to the high memory requirements of these computations and sparsity of the weight gradients, we only compute them for the biases. To estimate the gradient variances we use 1000 Monte Carlo samples. 
\subsection{Additional Results for DVAEs}
\subsubsection{Variance Reduction}
In \Cref{fig:dvae_variance} we present additional results on the practical variance reduction that VarGrad induces in the two layer linear DVAE. Here, we compare with various other estimators from the literature. VarGrad achieves considerable variance reduction over the adaptive (RELAX) and non-adaptive (Controlled Reinforce) model-agnostic estimators. Structured adaptive estimators such as Dynamic REBAR and RELAX + REBAR start with a higher variance at the beginning of optimisation, which reduces towards the end. ARM, which uses antithetic sampling, achieves the most reduction; however, it is only applicable to models with Bernoulli latent variables. Notably, the extra variance reduction seen in some of the methods does not translate to better optimisation performance on this example as seen in \Cref{appx:vae-results}.
\FloatBarrier
\begin{figure}[h!]
\centering
\includegraphics[width=\linewidth]{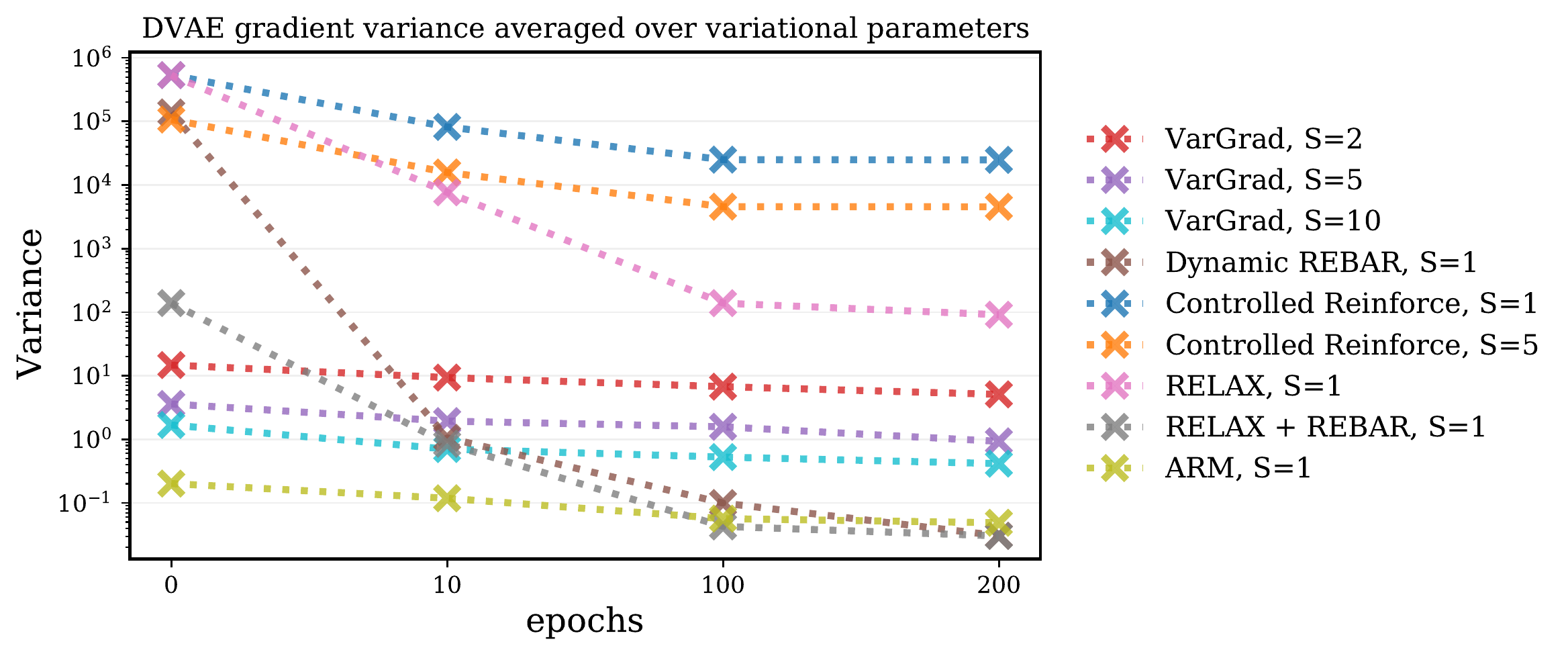}
\caption{Estimates of the gradient variance of the DVAE at 4 points during the optimisation for different gradient estimators. The plot compares VarGrad to Reinforce with score function control variates \citep{ranganath2014black}, dynamic \textsc{rebar} \citep{tucker2017rebar}, \textsc{relax}, \textsc{relax} + \textsc{rebar} \citep{grathwohl2018backpropagation} and ARM \citep{yin2018arm}. The number of samples used to compute each gradient estimator is given in the figure legend.}
\label{fig:dvae_variance}
\end{figure}

\FloatBarrier
\subsubsection{Performance in Optimisation}
\label{appx:vae-results}
In this section we present additional results on training the DVAE with VarGrad. \Cref{fig:omniglot_appendix_1} replicates \Cref{fig:omniglot_main} with a longer run-time. \Cref{fig:omniglot_appendix_2} and \Cref{fig:omniglot_appendix_3} show the optimisation traces for different Adam learning rates.

\begin{figure}[h]
\centering
\includegraphics[width=\linewidth]{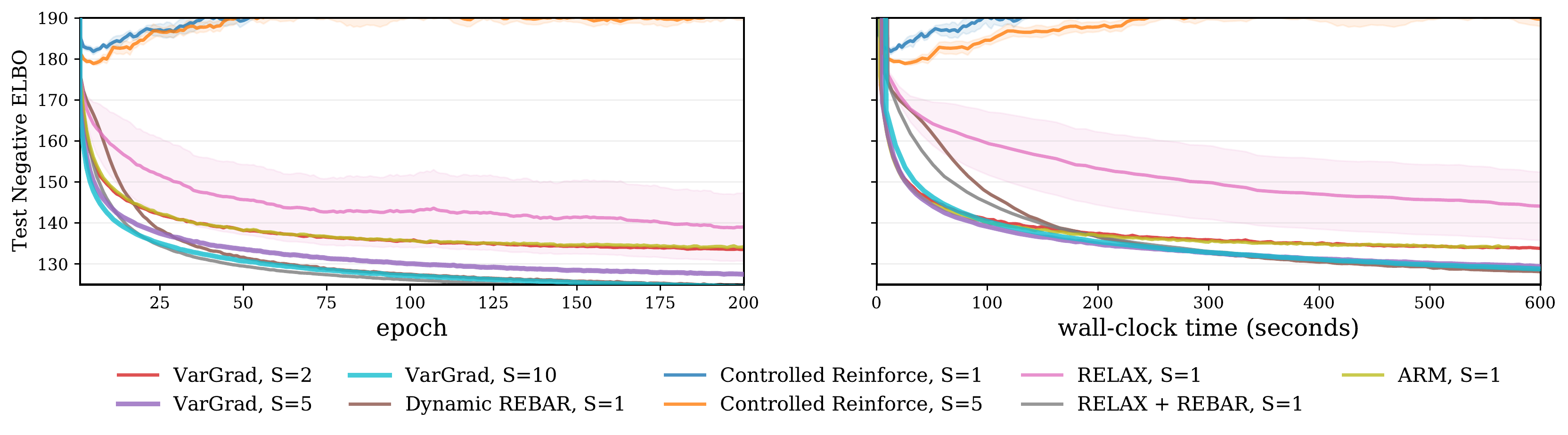}
\caption{Optimisation trace versus epoch (left) and wall-clock time (right) for a two-layer linear \gls{DVAE} on a fixed binarisation of Omniglot, trained with Adam with a learning rate of 0.001. The plot compares VarGrad to Reinforce with score function control variates \citep{ranganath2014black}, dynamic \textsc{rebar} \citep{tucker2017rebar}, \textsc{relax}, \textsc{relax} + \textsc{rebar} \citep{grathwohl2018backpropagation} and ARM \citep{yin2018arm}. The number of samples used to compute each gradient estimator is given in the figure legend. The results here are identical to the ones in \Cref{fig:omniglot_main} but with a longer run-time.}
\label{fig:omniglot_appendix_1}
\end{figure}

\begin{figure}[h]
\centering
\includegraphics[width=\linewidth]{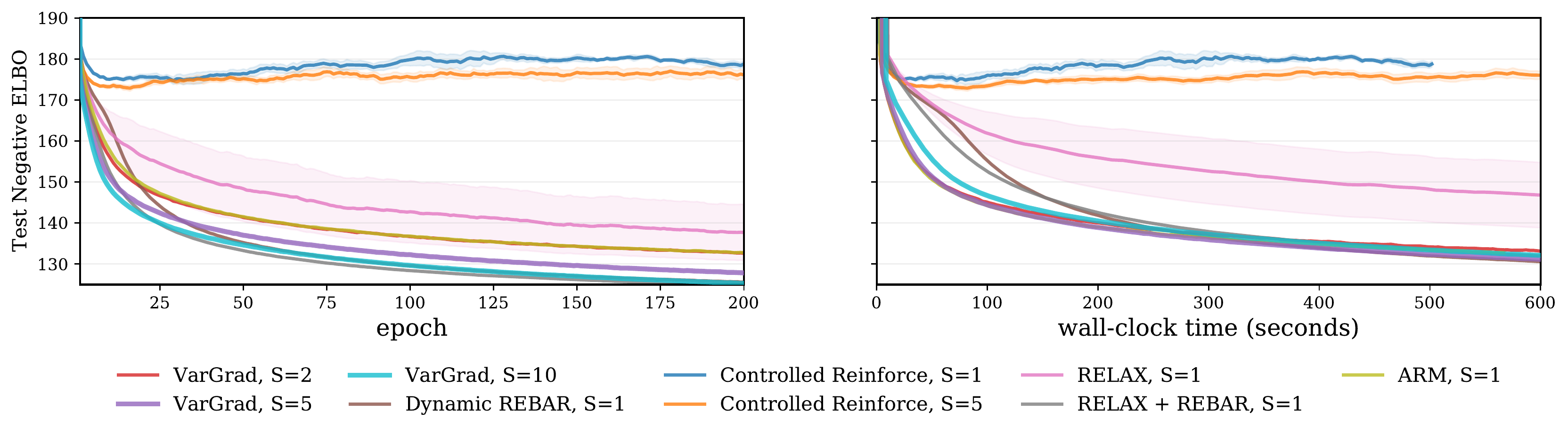}
\caption{Optimisation trace versus epoch (left) and wall-clock time (right) for a two-layer linear \gls{DVAE} on a fixed binarisation of Omniglot, trained with Adam with a learning rate of 0.0005. The plot compares VarGrad to Reinforce with score function control variates \citep{ranganath2014black}, dynamic \textsc{rebar} \citep{tucker2017rebar}, \textsc{relax}, \textsc{relax} + \textsc{rebar} \citep{grathwohl2018backpropagation} and ARM \citep{yin2018arm}. The number of samples used to compute each gradient estimator is given in the figure legend.}
\label{fig:omniglot_appendix_2}
\end{figure}

\begin{figure}[h]
\centering
\includegraphics[width=\linewidth]{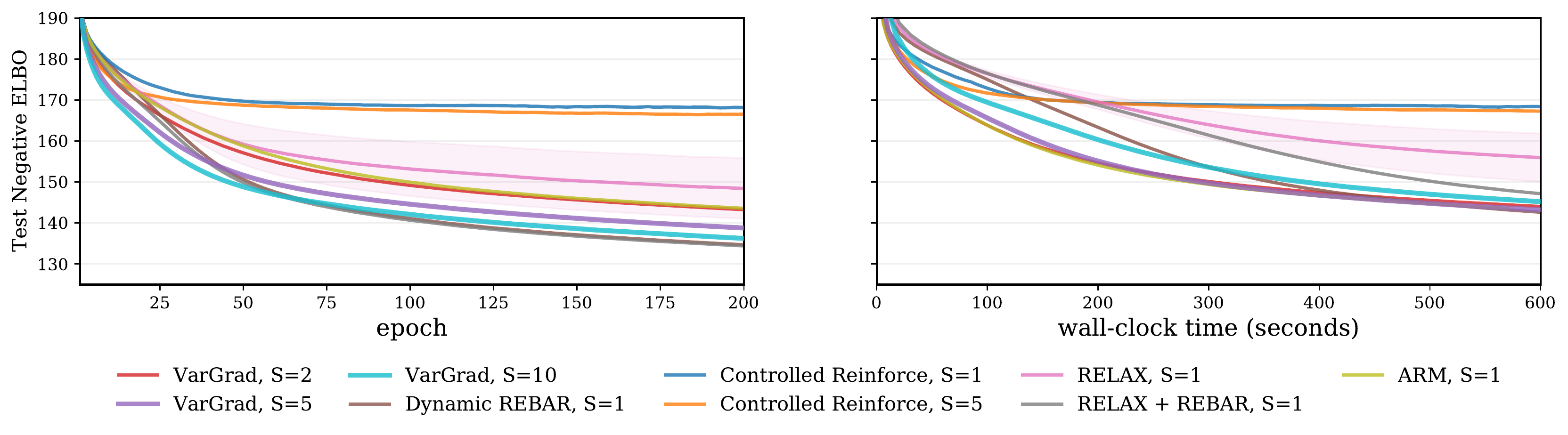}
\caption{Optimisation trace versus epoch (left) and wall-clock time (right) for a two-layer linear \gls{DVAE} on a fixed binarisation of Omniglot, trained with Adam with a learning rate of 0.0001. The plot compares VarGrad to Reinforce with score function control variates \citep{ranganath2014black}, dynamic \textsc{rebar} \citep{tucker2017rebar}, \textsc{relax}, \textsc{relax} + \textsc{rebar} \citep{grathwohl2018backpropagation} and ARM \citep{yin2018arm}. The number of samples used to compute each gradient estimator is given in the figure legend.}
\label{fig:omniglot_appendix_3}
\end{figure}

\FloatBarrier
\newpage
\section{Gaussians}\label{app:gaussians}

In the case when $q(z)$ and $p(z|x)$ are (diagonal) Gaussians we can gain some intuition on the performance of VarGrad by computing the relevant quantities analytically. The principal insights obtained from the examples presented in this section can be summarised as follows: Firstly, in certain scenarios the Reinforce estimator does indeed exhibit a lower variance in comparison with VarGrad (although the advantage is very modest and only materialises for a restricted set of parameters). This finding illustrates that the conditions in \Cref{eq:comparison bound} and \Cref{eq:better variance} (the latter referring to $S \ge S_0$) cannot be dropped without replacement from the formulation of \Cref{prop: variance comparison log-var vs. reinforce}. Secondly, in line with the results from \Cref{sec:experiments}, the relative error   $\delta_i^{\textrm{CV}}/\bE[a^{\textnormal{VarGrad}}]$ decreases with increased dimensionality. Moreover, the variance associated to computing the optimal control variate coefficients $a^*$ is significant and increases considerably with the number of latent variables.

\subsection{Comparing the Variances of Reinforce and VarGrad}

\label{appendix: Covariance terms for diagonal Gaussians}

In order to understand when the variance of VarGrad is smaller than the variance of the Reinforce estimator we first consider the one-dimensional Gaussian case $q(z) = \mathcal{N}(z; \mu, \sigma^2)$ and $p(z|x) = \mathcal{N}(z; \widetilde{\mu}, \widetilde{\sigma}^2)$ and analyse the derivative w.r.t. $\mu$. A lengthy calculation shows that
\begin{subequations}
\label{eqn: Variance difference 1d Gaussian}
\begin{align}
\Delta_{\Var}&(\mu, \widetilde{\mu}, \sigma^2, \widetilde{\sigma}^2, S) := \Var(\widehat{g}_{\text{Reinforce},\mu}) - \Var(\widehat{g}_{\text{VarGrad},\mu}) \\
\label{eqn: Variance difference 1d Gaussian - formula}
&=\frac{1}{4 S {\sigma}^4\widetilde{\sigma}^2}\left((\mu - \widetilde{\mu})^4 + 2(\mu - \widetilde{\mu})^2\left(\frac{3S - 7}{S-1} {\sigma}^2 - 3\widetilde{\sigma}^2\right) +\frac{5S - 7}{S-1}\left( {\sigma}^2 - \widetilde{\sigma}^2\right)^2\right) \\
\label{eqn: Variance difference 1d Gaussian - approx}
&\approx \frac{1}{4 S {\sigma}^4\widetilde{\sigma}^2} \left( \Delta_\mu^4 + 6 \Delta_\mu^2 \Delta_{\sigma^2} + 5 \Delta_{\sigma^2}^2 \right),
\end{align}
\end{subequations}
where the last line holds for large $S$ with $\Delta_\mu:=\mu - \widetilde{\mu}$ and $\Delta_{\sigma^2} := {\sigma}^2 - \widetilde{\sigma}^2$.

For an illustration, let us vary the above parameters. First, let us fix $\sigma^2 = \widetilde{\sigma}^2 = 1$. We note from \Cref{eqn: Variance difference 1d Gaussian - approx} that in this case we expect VarGrad to have lower variance regardless of $\Delta_\mu$ as long as $S$ is large enough. In Figure \ref{fig: variance difference varying mu} we see that this is in fact the case, however a different result can be observed for small $S$, which is again in accordance with \Cref{eqn: Variance difference 1d Gaussian - formula}.

\begin{figure}[t]
\centering
\includegraphics[width=1\linewidth]{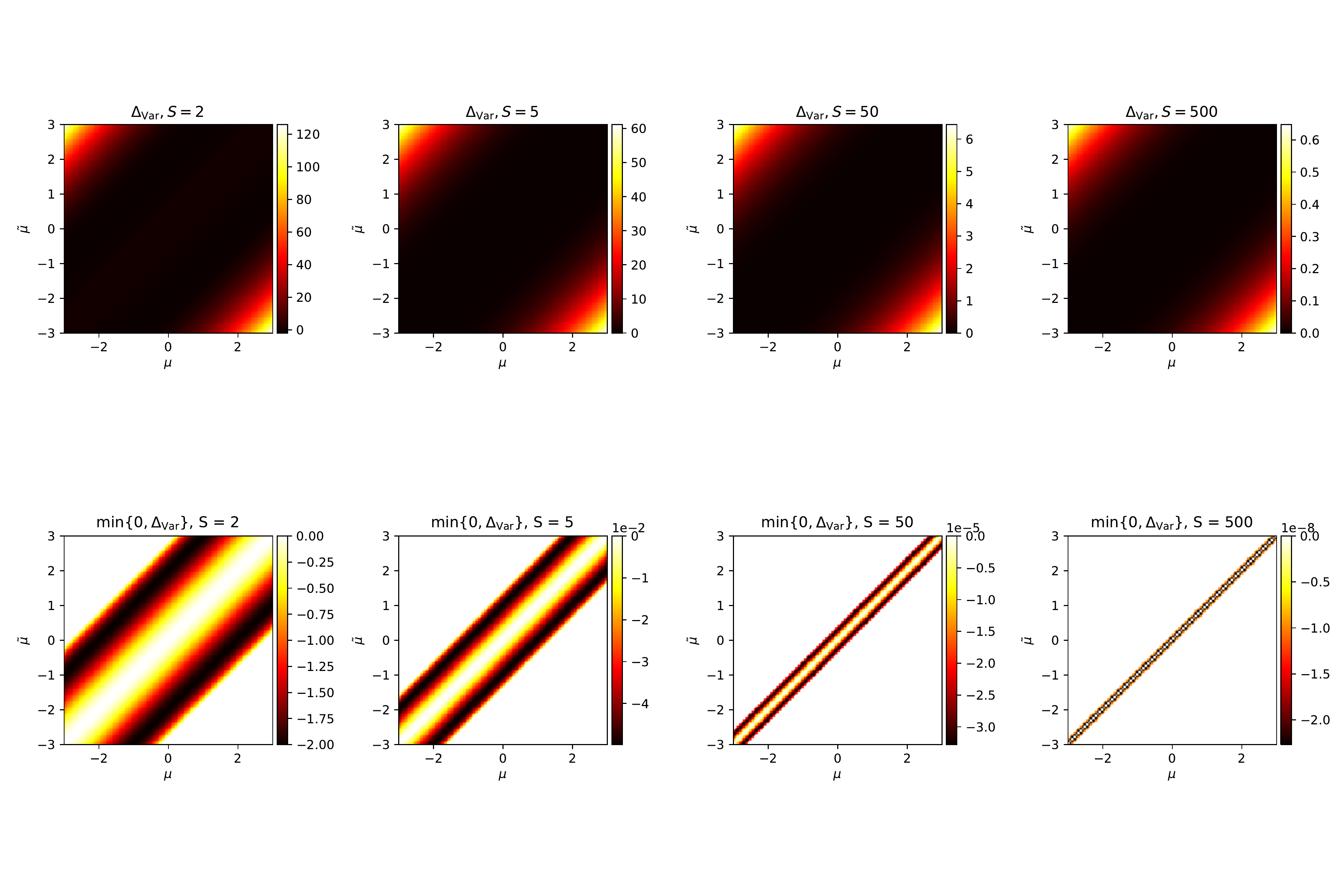}
\caption{We compare the variance of the reinforce estimator with the variance of VarGrad. VarGrad is often better even for small $S$ -- for large $S$ this can be guaranteed with Proposition \ref{prop: variance comparison log-var vs. reinforce}.}
\label{fig: variance difference varying mu}
\end{figure}

Next, we consider arbitrary $\sigma^2$ and $\widetilde{\sigma}^2$, but fixed $\mu=1,\widetilde{\mu}=2$. In Figure \ref{fig: variance difference varying sigma} we observe that the variance of VarGrad is smaller for most values of $\sigma^2$ and $\widetilde{\sigma}^2$. However, even for large $S$ there remains a region where the Reinforce estimator is superior. In fact, one can compute the condition for this to happen to be $\Delta_{\sigma
^2} \in \left[-\Delta_{\mu}^2, -\frac{1}{5}\Delta_{\mu}^2 \right]$, which can be compared with the condition in \Cref{eq:comparison bound} in Proposition \ref{prop: variance comparison log-var vs. reinforce}.

\begin{figure}[t]
\centering
\includegraphics[width=1\linewidth]{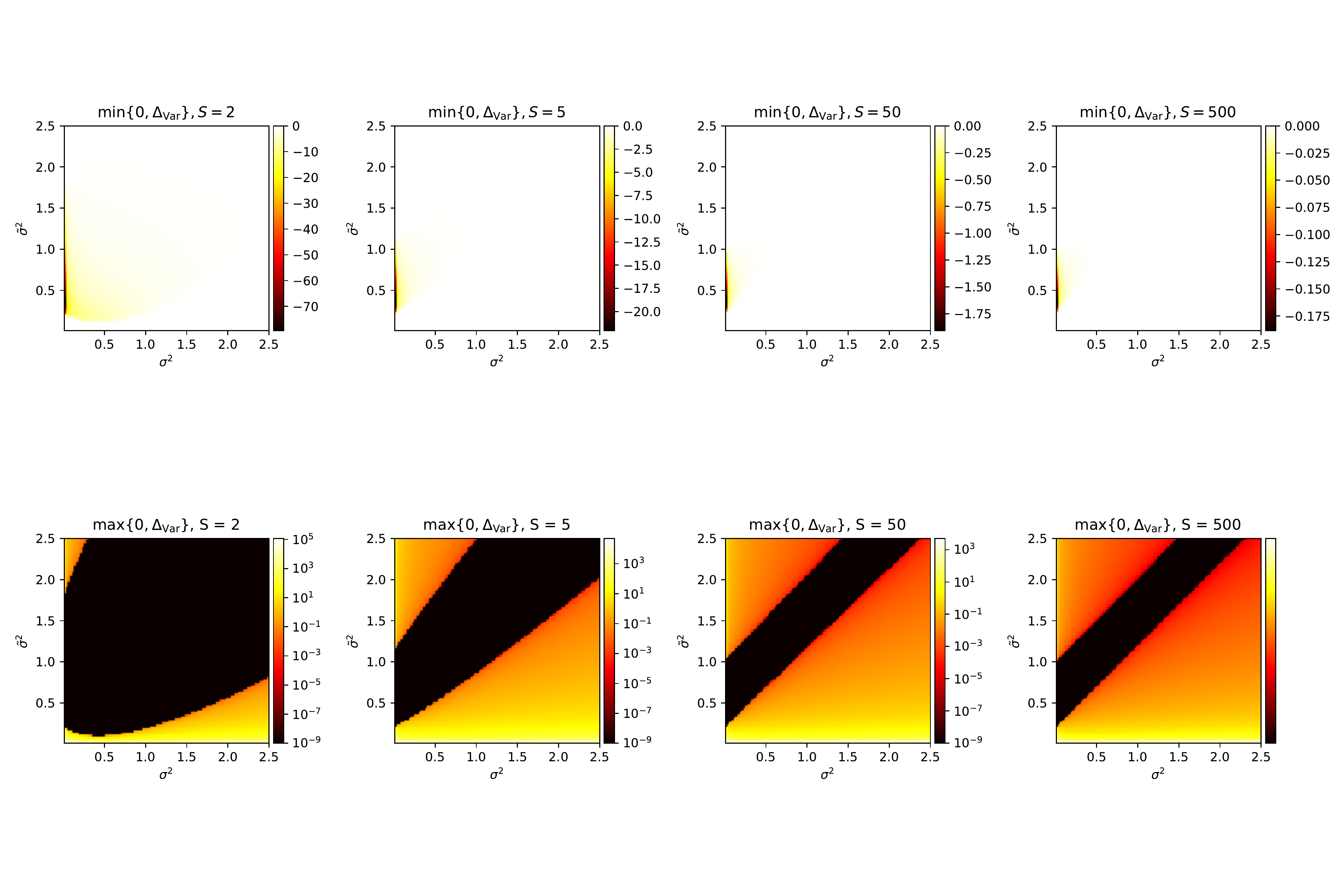}
\caption{Variance comparison with varying $\sigma^2$ and $\widetilde{\sigma}^2$. VarGrad only wins outside a certain region, however, if so, then potentially by orders of magnitude.}
\label{fig: variance difference varying sigma}
\end{figure}

In Figure \ref{fig: Gaussian-delta-mu-sigma} we display the variance differences $\Delta_{\Var}$ as functions of $\Delta_{\mu}$ and $\Delta_{\sigma^2}$, approximated according to \Cref{eqn: Variance difference 1d Gaussian - approx}, for the same fixed values as before and see that they are bounded from below, but not from above.

\begin{figure}[t]
\centering
\includegraphics[width=1\linewidth]{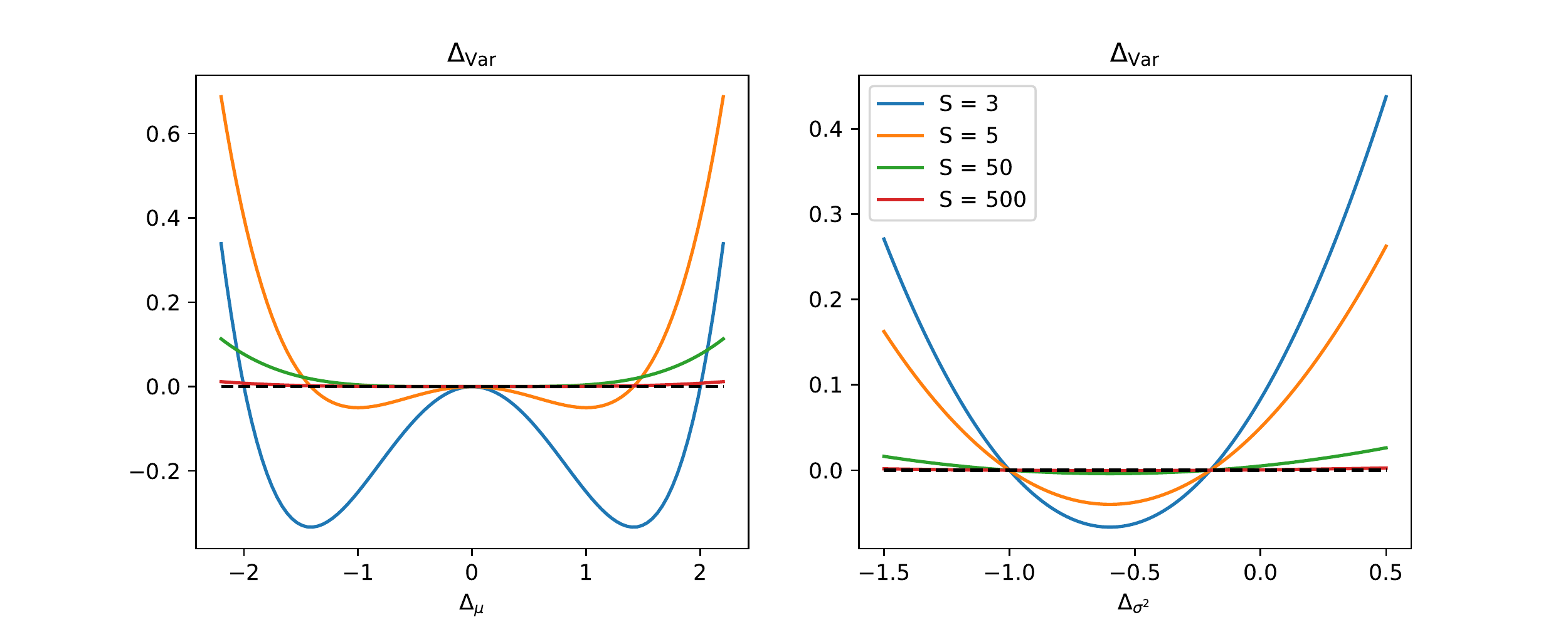}
\caption{Variance differences of the reinforce estimator and VarGrad with varying $\Delta_\mu$ and $\Delta_{\sigma^2}$ for different sample sizes $S$.}
\label{fig: Gaussian-delta-mu-sigma}
\end{figure}

For a $D$-dimensional Gaussian it is hard to compute the condition from \Cref{eq:comparison bound} in full generality, but we can derive the following stronger criterion that can guarantee better performance of VarGrad when assuming that $\mathrm{ELBO}(\phi) \le 0$ (which for instance holds in the discrete-data setting). 

\begin{lemma}
Assume $\mathrm{ELBO}(\phi) \le 0$ and 
\begin{equation}
\label{eqn: stronger criterion for VarGrad vs. Reinforce}
    \Cov\left(f_\phi, \left(\partial_{\phi_i} \log q_\phi \right)^2\right) > 0.
\end{equation}
Then there exists $S_0 \in \mathbb{N}$ such that 
\begin{equation}
\Var \left(\widehat{g}_{\text{VarGrad}, i}(\phi)\right) \leq \Var \left(\widehat{g}_{ \text{Reinforce}, i}(\phi)\right), \quad \quad \textnormal{for all} \quad S \ge S_0.
\end{equation}
\end{lemma}
\begin{proof}
With $\mathrm{ELBO}(\phi) \le 0$ we have
\begin{align}
    \Cov\left(f_\phi, \left(\partial_{\phi_i} \log q_\phi \right)^2\right)&\le \E_{q_\phi}\left[f_\phi \left(\partial_{\phi_i} \log q_\phi \right)^2 \right] - \frac{1}{2} \E_{q_\phi}\left[f_\phi\right]\E_{q_\phi}\left[ \left(\partial_{\phi_i} \log q_\phi \right)^2 \right] \\
    &= \E_{q_\phi}\left[\left(\partial_{\phi_i} \log q_\phi \right)^2\right]\left(\delta^{\text{CV}}_i - \frac{1}{2}\mathrm{ELBO}(\phi)\right).
\end{align}
If now
\begin{equation}
    \Cov\left(f_\phi, \left(\partial_{\phi_i} \log q_\phi \right)^2\right) > 0,
\end{equation}
then also
\begin{equation}
    \delta^{\text{CV}}_i - \frac{1}{2}\mathrm{ELBO}(\phi) > 0,
\end{equation}
and the statement follows by Proposition \ref{prop: variance comparison log-var vs. reinforce}.
\end{proof}

The condition from \Cref{eqn: stronger criterion for VarGrad vs. Reinforce} gives another guarantee for VarGrad having smaller variance than the Reinforce estimator. However, we note that the converse statement is not necessarily true, i.e. if the condition does not hold, VarGrad can still be better. The advantage of \Cref{eqn: stronger criterion for VarGrad vs. Reinforce}, however, is that it can be verified more easily in certain settings, as for instance done for $D$-dimensional diagonal Gaussians in the following lemma.

\begin{lemma}[Covariance term for diagonal Gaussians]
\label{Cov(A,B^2) for diagonal Gaussians}
Let $q(z)$ and $p(z|x)$ be diagonal $D$-dimensional Gaussians with means $\mu$ and $\widetilde{\mu}$ and covariance matrices $\Sigma = \operatorname{diag}(\sigma_1^2, \dots, \sigma_D^2)$ and $\widetilde{\Sigma} = \operatorname{diag}(\widetilde{\sigma}_1^2, \dots, \widetilde{\sigma}_D^2)$. Then 
\begin{equation}
    \Cov_{q_{\phi}} \left( f_\phi ,\left( \partial_{\phi_k} \log q_\phi \right)^2\right) = \frac{1}{\widetilde{\sigma}_k^2} -  \frac{1}{\sigma_k^2}
\end{equation}
for $k \in \{1, \dots, D \}$ and 
\begin{equation}
    \Cov_{q_{\phi}} \left( f_\phi ,\left( \partial_{\phi_k} \log q_\phi \right)^2\right) = \frac{1}{{\sigma}_k^2}\left(\frac{1}{\widetilde{\sigma}_k^2} - \frac{1}{{\sigma}_k^2} \right)
\end{equation}
for $k \in \{D + 1, \dots, 2D \}$ with $\phi = (\mu_1, \dots, \mu_D, \sigma_1^2, \dots, \sigma_D^2)^\top$
and
\begin{equation}
    \Cov_{q_{\phi}} \left( f_\phi ,\left( \partial_{\phi_k} \log q_\phi \right)^2\right) = \frac{1}{\widetilde{\sigma}_k^2} - \frac{1}{{\sigma}_k^2} 
\end{equation}
for $k \in \{D + 1, \dots, 2D \}$ with $\phi = (\mu_1, \dots, \mu_D, \log \sigma_1^2, \dots, \log \sigma_D^2)^\top$.
\end{lemma}
\begin{proof}
We compute 
\begin{equation}
     f_\phi  = -\frac{1}{2}\sum_{i=1}^D \log \left(\frac{\sigma_i^2}{\widetilde{\sigma}_i^2}\right) - \frac{1}{2}\sum_{i=1}^D\frac{(z_i-\mu_i)^2}{\sigma_i^2}+\frac{1}{2}\sum_{i=1}^D\frac{(z_i-\widetilde{\mu}_i)^2}{\widetilde{\sigma}^2_i}
\end{equation}
and
\begin{equation}
   \partial_{\mu_k} \log q_\phi = \frac{z_k-\mu_k}{\sigma_k^2}. 
\end{equation}
We again use the short-cuts
\begin{equation}
A = f_\phi(z), \qquad B = \left(\partial_{\mu_k} \log q_\phi \right)(z),
\end{equation}
and obtain

\begin{align}
    &\Cov_{q_\phi}(A,B^2) = \E_{q_\phi}\left[AB^2\right] - \E_{q_\phi}[A]\E_{q_\phi}\left[B^2\right] \\
    &= \E_{q_\phi}\left[\left(-\frac{1}{2}\sum_{i=1}^D \log \left(\frac{\sigma_i^2}{\widetilde{\sigma}_i^2}\right) - \frac{1}{2}\sum_{i=1}^D\frac{(z_i-\mu_i)^2}{\sigma_i^2}+\frac{1}{2}\sum_{i=1}^D\frac{(z_i-\widetilde{\mu}_i)^2}{\widetilde{\sigma}^2_i}\right)\left( \frac{z_k-\mu_k}{\sigma_k^2}\right)^2\right] \\
    &\quad - \E_{q_\phi}\left[\left(-\frac{1}{2}\sum_{i=1}^D \log \left(\frac{\sigma_i^2}{\widetilde{\sigma}_i^2}\right) - \frac{1}{2}\sum_{i=1}^D\frac{(z_i-\mu_i)^2}{\sigma_i^2}+\frac{1}{2}\sum_{i=1}^D\frac{(z_i-\widetilde{\mu}_i)^2}{\widetilde{\sigma}^2_i}\right)\right]\E_{q_\phi}\left[\left(\frac{z_k-\mu_k}{\sigma_k^2}\right)^2\right] \\
    &= -\frac{1}{2}\left(\frac{3}{\sigma_k^2} + \frac{D-1}{\sigma_k^2}\right) + \frac{1}{2}\left(\frac{1}{{\sigma}_k^2}\sum_{\substack{i=1 \\ i \neq k}}^D \frac{\sigma_i^2 + (\mu_i - \widetilde{\mu}_i)^2}{\widetilde{\sigma}_i^2 } + \frac{1}{\sigma_k^2 \widetilde{\sigma}_k^2}\left(3\sigma_k^2 + (\mu_k - \widetilde{\mu}_k)^2 \right) \right)\\
    &\quad - \left( -\frac{D}{2} + \frac{1}{2} \sum_{i=1}^D \frac{\sigma_i^2 + (\mu_i - \widetilde{\mu}_i)^2}{\widetilde{\sigma}_i^2 }\right)\frac{1}{\sigma_k^2} \\
    &= -\frac{1}{\sigma_k^2} +\frac{1}{2\sigma_k^2\widetilde{\sigma}_k^2}\left(3\sigma_k^2 + (\mu-\widetilde{\mu})^2 \right) - \frac{1}{2\sigma_k^2\widetilde{\sigma}_k^2}\left(\sigma_k^2 + (\mu-\widetilde{\mu})^2 \right) \\
    &= \frac{1}{\widetilde{\sigma}_k^2} - \frac{1}{\sigma_k^2}.
\end{align}
For the terms with the partial derivative w.r.t. $\sigma_k^2$ we first note that
\begin{equation}
\label{eqn: partial_sigma^2_log_q}
\partial_{\sigma_k^2} \log q_\phi = -\frac{1}{2 \sigma_k^2} + \frac{(z_k-\mu_k)^2}{2\sigma_k^4} = \frac{1}{\sigma_k^2} \partial_{\log \sigma_k^2} \log q_\phi.
\end{equation}
We compute
\begin{align}
    \E_{q_\phi}\left[ AB^2\right] &= \E_{q_\phi}\Bigg[\frac{(z_k-\mu_k)^2}{4\sigma_k^6} \sum_{i=1}^D \frac{(z_i-\mu_i)^2}{\sigma_i^2} - \frac{(z_k-\mu_k)^4}{8\sigma_k^8}\sum_{i=1}^D \frac{(z_i-\mu_i)^2}{\sigma_i^2} \\
    &\quad\quad\qquad-\frac{(z_k - \mu_k)^2}{4 \sigma_k^6}\sum_{i=1}^D \frac{(z_i-\widetilde{\mu}_i)^2}{\widetilde{\sigma}_i^2} + \frac{(z_k-\mu_k)^4}{8 \sigma_k^8}\sum_{i=1}^D \frac{(z_i-\widetilde{\mu}_i)^2}{\widetilde{\sigma}_i^2} \Bigg] \\
    &= -\frac{1}{8\sigma_k^4}(D + 8) + \frac{1}{8\sigma_k^4 \widetilde{\sigma}_k^2} (9 \sigma_k^2 + (\mu_k - \widetilde{\mu}_k)^2) + \frac{1}{8\sigma_k^4} \sum_{\substack{i=1 \\ i \neq k}}^d \frac{\sigma_i^2 + (\mu_i - \widetilde{\mu_i})^2}{\widetilde{\sigma}_i^2},
\end{align}
and similarly
\begin{align}
    \E_{q_\phi}[A]\E_{q_\phi}\left[ B^2\right] = -\frac{D}{8 \sigma_k^4} + \frac{1}{8 \sigma_k^4 \widetilde{\sigma}_k^2}\left(\sigma_k^2 + (\mu_k - \widetilde{\mu}_k)^2 \right) + \frac{1}{8\sigma_k^4} \sum_{\substack{i=1 \\ i \neq k}}^D  \frac{\sigma_i^2 + (\mu_i - \widetilde{\mu}_i)^2}{\widetilde{\sigma}_i^2}.
\end{align}
We therefore get the result by again computing $\Cov_{q_\phi}(A, B^2) =  \E_{q_\phi}\left[ AB^2\right] - \E_{q_\phi}[A]\E_{q_\phi}\left[ B^2\right]$. The partial derivative w.r.t. $\log \sigma_k^2$ can be recovered from \Cref{eqn: partial_sigma^2_log_q}.
\end{proof}

\subsection{Optimal Control Variates in the Gaussian Case}

In the diagonal Gaussian case we can also easily analytically compute the optimal control variate coefficients from \Cref{eqn: optimal control variate}, along the lines of the proof of Lemma \ref{Cov(A,B^2) for diagonal Gaussians}. Our setting is again $q(z) = \mathcal{N}(z; \mu, \Sigma)$, $p(z|x) = \mathcal{N}(z; \widetilde{\mu}, \widetilde{\Sigma})$ with $\Sigma = \operatorname{diag}\left(\sigma_1^2, \dots, \sigma_D^2 \right)$, $\widetilde{\Sigma} = \operatorname{diag}\left(\widetilde{\sigma}_1^2, \dots, \widetilde{\sigma}_D^2 \right)$. In Figure \ref{fig: diag Gaussians optimal CV} we plot the variances of four different gradient estimators with varying sample size $S$, namely $\widehat{g}_\text{Reinforce}$, $\widehat{g}_\text{VarGrad}$, as well as the Reinforce estimator augmented with the optimal control variate, once computed analytically and once sampled using $S$ samples. The variance depends on the mean and the covariance matrix; here we choose $\mu_i = 3$, $\sigma_i^2 = 3$, $\widetilde{\mu}_i = 1$, $\widetilde{\sigma}_i^2 = 1$ for all $i\in\{1, \dots, D\}$.
\begin{figure}[t]
\centering
\includegraphics[width=1.0\linewidth]{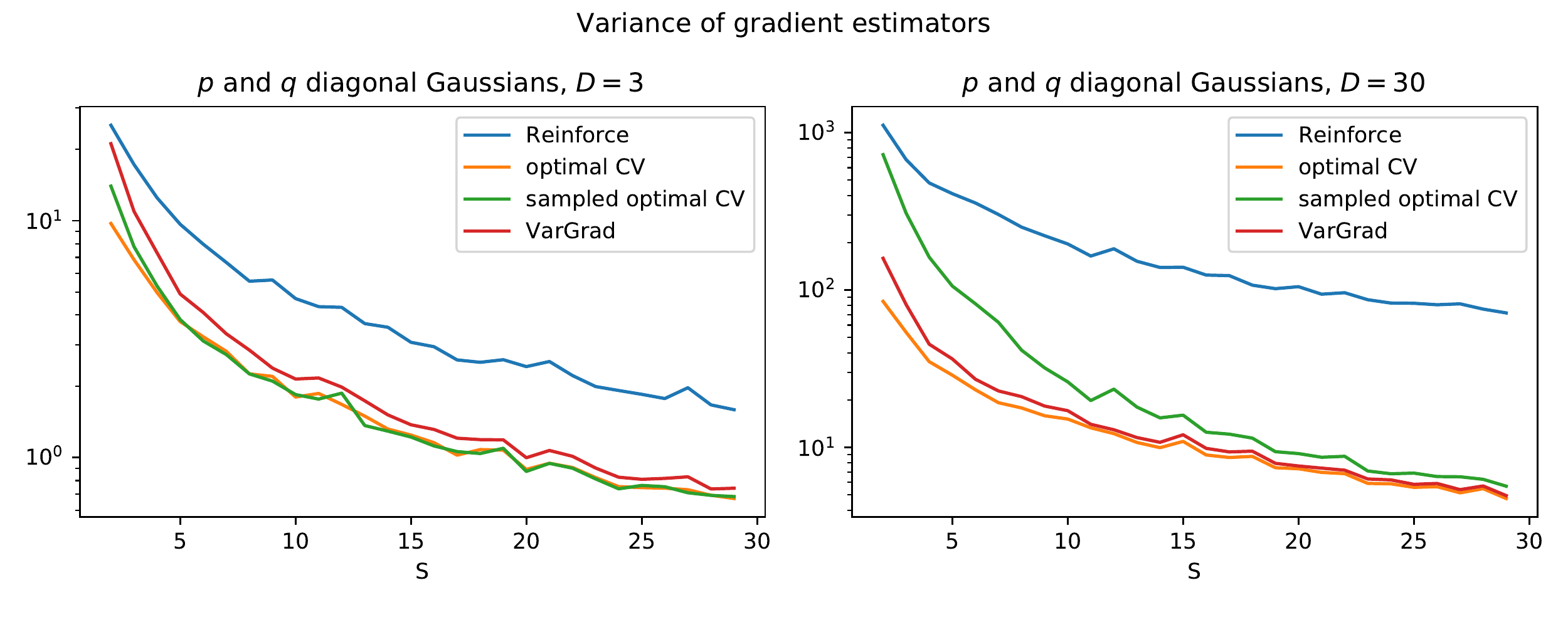}
\caption{Comparion of the variances of the different gradient estimators $\widehat{g}_\text{Reinforce}$, $\widehat{g}_\text{VarGrad}$, as well as the reinforce estimator with the optimal control variate coefficient, once computed analytically and once sampled with $S$ samples, for dimensions $D=3$ and $D=30$.}
\label{fig: diag Gaussians optimal CV}
\label{fig: computed-vs-sampled-CV-Gaussian}
\end{figure}
We observe that the VarGrad estimator is close to the analyitcal optimal control variate, and that the sampled optimal control variate performs significantly worse in a small sample size regime. These observations get more pronounced in higher dimensions and indicate that the variance of the sampled optimal control variate can itself be high, showing that using it might not always be beneficial in practice. 

Let us additionally investigate the optimal control variate correction term $\delta_i^\text{CV}$ as defined in \Cref{eqn: delta_CV} for $D$-dimensional Gaussians $q(z)$ and $p(z|x)$ as considered above. In Figure \ref{fig: delta-vs-KL-diag-Gaussians} we display the variances, means and relative errors of $\delta_i^\text{CV}$ and $a^{\text{VarGrad}} = \bar{f}_\phi$ and realise that indeed the ratio of those two converges to zero when $D$ gets larger. Furthermore we notice that the relative error of $\delta^\text{CV}$ increases with the dimension, explaining the difficulties when estimating the optimal control variate coefficients from samples. Finally, we plot a histogramm of $\delta^{\text{CV}}_i$ (varying across $i$) in \Cref{fig: delta-histogram-diagonal-Gaussians}, showing that $\delta^{\text{CV}}_i$ is small in comparison to $\E\left[a^{\text{VarGrad}}\right]$ and distributed around zero. 

\begin{figure}[t]
\centering
\includegraphics[width=1\linewidth]{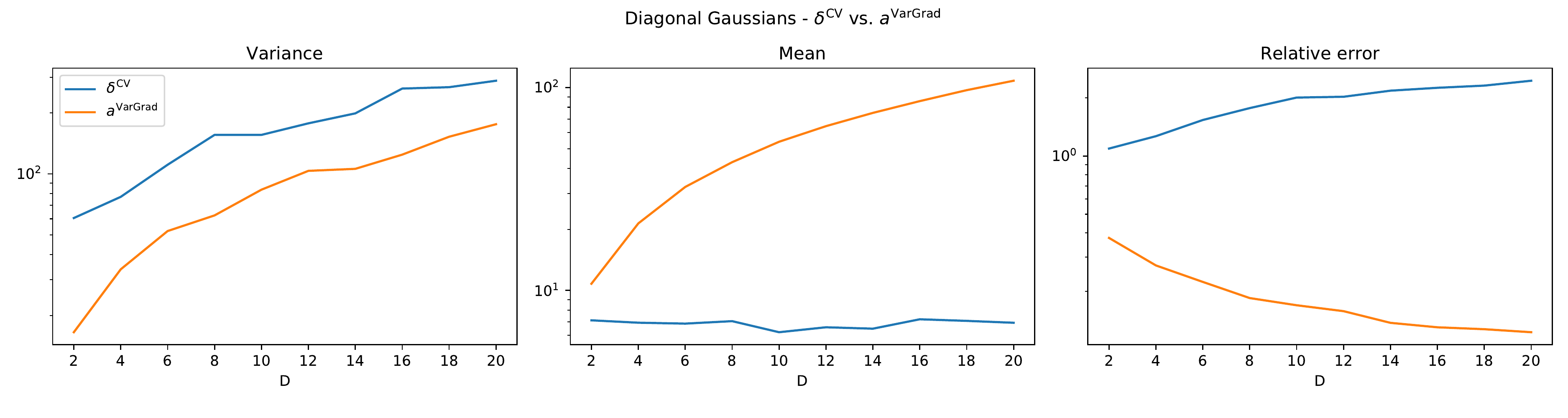}
\caption{Mean, variance and relative errors associated to the two contributions to the optimal control variate coefficient, $\delta_i^\text{CV}$ and $a^{\text{VarGrad}} = \bar{f}_\phi$.}
\label{fig: delta-vs-KL-diag-Gaussians}
\end{figure}

\begin{figure}[t]
\centering
\includegraphics[width=0.5\linewidth]{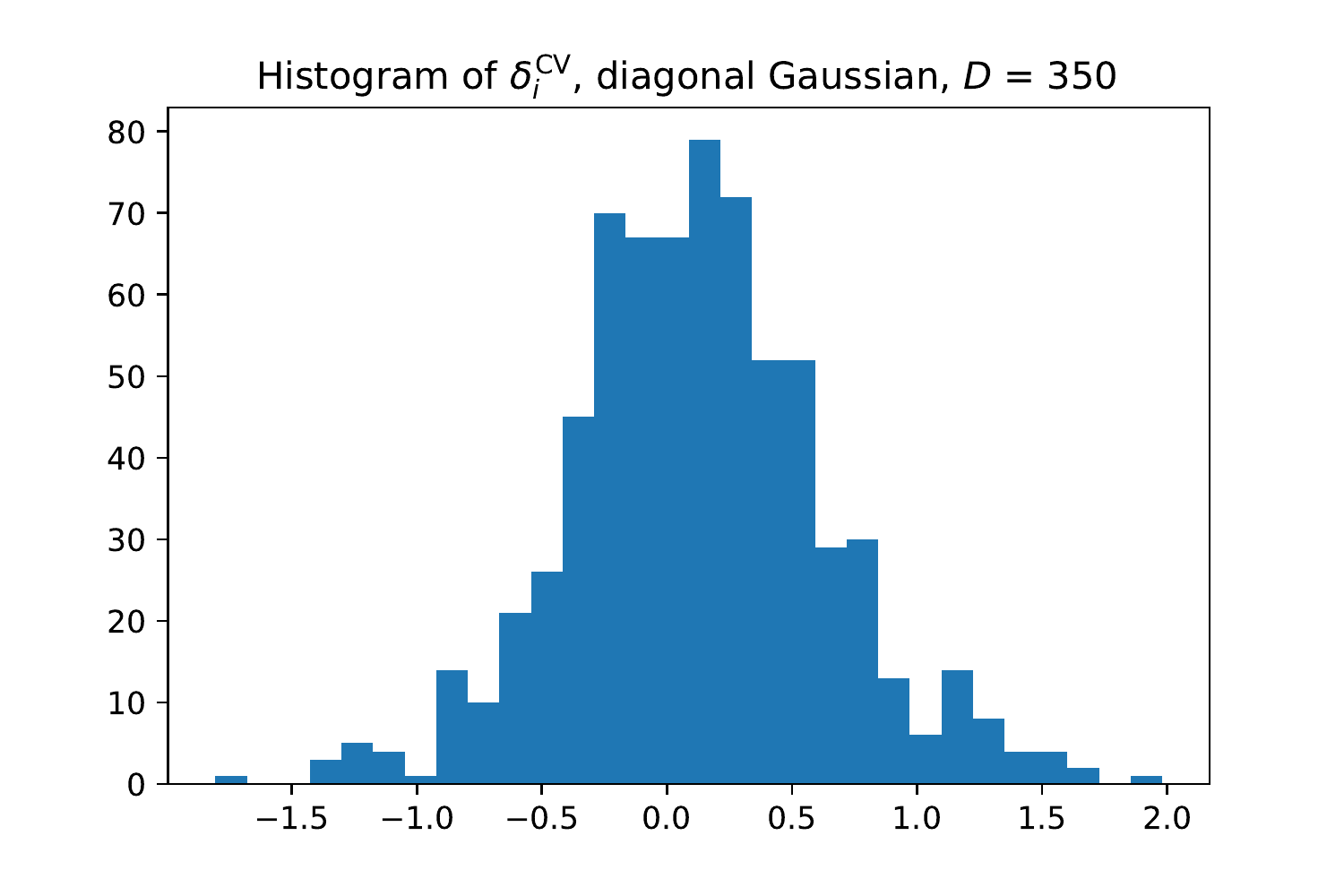}
\caption{The histogram of $\delta_i^\text{CV}$ shows that it is usually rather small in comparison to $\E\left[a^{\text{VarGrad}}\right]$, which is roughly $700$ here, and that it fluctuates around zero.}
\label{fig: delta-histogram-diagonal-Gaussians}
\end{figure}

\section{Connections to Other Divergences}
\label{appendix: other divergences}

The Reinforce gradient estimator from \Cref{eq:score_function} can as well be derived from the \textit{moment loss}
\begin{equation}
    \mathcal{L}_r^{\text{moment}}(q_\phi(z)\| p(z|x)) = \frac{1}{2} \E_{r(z)}\left[\log^2 \left(\frac{q_\phi(z)}{p(z|x)}\right)\right],
\end{equation}
namely
    \begin{equation}
        \nabla_\phi \mathcal{L}_r^{\text{moment}}(q_\phi(z)\| p(z|x)) \Big|_{r = {q_\phi}} = \E_{q_\phi}\left[\log \left(\frac{q_\phi(z)}{p(z|x)}\right) \nabla_\phi \log q_{\phi}(z)  \right].
    \end{equation}
    
In the log-variance loss, one can omit the logarithm to obtain the \textit{variance loss}
\begin{equation}
    \mathcal{L}_r^{\Var}(p(z|x) \| q_\phi(z)) = \frac{1}{2} \Var_{r(z)}\left(\frac{p(z|x)}{q_\phi(z)} \right) = \frac{1}{2} \E_{r(z)}\left[\left(\frac{p(z|x)}{q_\phi(z)} \right)^2 - 1 \right],
\end{equation}
which with $r = q_\phi$ coincides with the $\chi^2$-divergence. The potential of using the latter in the context of variational inference was suggested in \cite{dieng2017variational}. We note that again one is in principle free in choosing $r(z)$, but that unlike the log-variance loss, this loss in not symmetric with respect to $q_\phi(z)$ and $p(z|x)$. An analysis in \cite{nusken2020solving} (for distributions on path space) however suggests that the variance loss (unlike the log-variance loss) scales unfavourably in high-dimensional settings in terms of the variance associated to standard Monte Carlo estimators.  
\end{document}